\DocumentMetadata{}
\documentclass[sigconf,10pt]{acmart}

\usepackage{url}
\usepackage{amsmath}
\usepackage{booktabs}
\usepackage{array}
\usepackage{xspace}
\usepackage{multicol}
\usepackage{array}
\usepackage{enumitem}
\usepackage{listings}
\usepackage{subcaption}
\usepackage{multirow}
\usepackage{mathtools}
\usepackage[capitalize,noabbrev,nameinlink]{cleveref}
\usepackage{algorithm}
\usepackage{algorithmic}

\crefname{lstlisting}{listing}{listings}
\Crefname{lstlisting}{Listing}{Listings}

\DeclareMathOperator*{\argmin}{\arg\min}

\newtheorem*{theorem*}{Theorem}
\newtheorem{theorem}{Theorem}[section]

\newtheorem{lemma}[theorem]{Lemma}
\newtheorem{proposition}{Proposition}[section]

\newtheorem{definition}{Definition}[section]
\newtheorem{remark}{Remark}[section]

\definecolor{dkgreen}{rgb}{0,0.6,0}
\definecolor{codegreen}{rgb}{0,0.6,0}
\definecolor{codegray}{rgb}{0.5,0.5,0.5}
\definecolor{codepurple}{rgb}{0.58,0,0.82}
\definecolor{backcolour}{rgb}{0.95,0.95,0.92}
\definecolor{purple}{RGB}{128,0,128}
\definecolor{indigo}{RGB}{75,0,130}
\definecolor{royalblue}{RGB}{65,105,225}
\definecolor{navy}{RGB}{0,0,128}
\definecolor{codebrown}{rgb}{0.6,0.6,0}

\lstdefinestyle{PyStyle}{
  language=Python,
  basicstyle=\ttfamily\footnotesize, 
  aboveskip = 0.05in,
  belowskip = 0.05in,
  breaklines=true,
  float=tp,
  floatplacement=tbp,
  frame=none,
  numbers=none,
  keepspaces=true,
  captionpos=b,
  showstringspaces=false,
  emph={MyClass,__init__},          %
  stringstyle=\color{deepgreen},
  emphstyle=\ttb\color{deepred},    %
  keywordstyle=\color{blue},
  commentstyle=\color{codegreen},
  morekeywords={self,def, for, sum, in, and}
}

\newif\ifcommenton
\commentonfalse %

\newcommand{\sys}{\mbox{Inshrinkerator}\xspace}
\newcommand{\paperTitle}{\sys: Compressing Deep Learning Training Checkpoints via Dynamic Quantization
}
\newcommand{\ours}{\mbox{Ours}\xspace}
\newcommand{\kmeans}{\mbox{K-Means}\xspace}

\newcommand{\quantizer}{\textsc{Quantizer}\xspace}
\newcommand{\qcm}{\textsc{Quantization Config Manager}\xspace}
\newcommand{\deltaE}{\textsc{Delta Encoder}\xspace}
\newcommand{\CNR}{\texttt{Adaptive}\xspace}
\newcommand{\QD}{\texttt{Uniform}\xspace}
\newcommand{\GOBO}{\texttt{KMeans}\xspace}

\newcommand{\hide}[1]{}

\newcommand{\PPP}[1]{
\vspace{0.05in}
\noindent{\textit{\IfEndWith{#1}{.}{#1}{#1.}}}
}

\newcommand{\sigsep}{\sigma\text{-}separable}

\newcommand{\oM}{M^{*}}
\newcommand{\tM}{\tilde{M}}
\newcommand{\wM}{\widehat{M}}
\newcommand{\otM}{\tilde{M}^{*}}
\newcommand{\tmu}{\tilde{\mu}}
\newcommand{\mui}{\mu_{M^{*}}(i)}
\newcommand{\tmui}{\tilde{\mu}_{\otM}(i)}
\newcommand{\tX}{\tilde{X}}
\newcommand{\tx}{\tilde{x}}
\newcommand{\txi}{\tilde{x}_{i}}
\newcommand{\loss}{\mathit{L}}
\newcommand{\qloss}{\mathit{L}_q}
\newcommand{\phiq}{\phi_q}

\newcommand*\rot{\rotatebox{90}}
\newcolumntype{s}{>{\centering\arraybackslash}m{1.5cm}}
\newcolumntype{x}{>{\centering\arraybackslash}m{2cm}}
\newcolumntype{L}{>{\centering\arraybackslash}m{3cm}}
\newcolumntype{y}{>{\centering\arraybackslash}m{0.7cm}{\hspace{2px}}}

\newcommand{\greencheck}{\textcolor{codegreen}{\checkmark}}
\newcommand{\redcross}{\textcolor{red}{$\times$}}
\newcommand{\orangecross}{\textcolor{blue}{$\times$}}
\newcommand{\minihead}[1]{{\vspace{.45em}\noindent\textbf{#1.} }}

\ifcommenton

\newcommand{\amey}[1]{\textcolor{codegreen}{[Amey: #1]}}
\newcommand{\alexey}[1]{\textcolor{indigo}{[AT: #1]}}
\newcommand{\kexin}[1]{\textcolor{codebrown}{[Kexin: #1]}}
\newcommand{\satwik}[1]{\textcolor{navy}{[SB: #1]}}

\else
\newcommand{\alexey}[1]{}
\newcommand{\amey}[1]{}
\newcommand{\kexin}[1]{}
\newcommand{\satwik}[1]{}
\fi

\title{\paperTitle}

\copyrightyear{2024}
\acmYear{2024}
\setcopyright{rightsretained}
\acmConference[SoCC '24]{ACM Symposium on Cloud Computing}{November 20--22, 2024}{Redmond, WA, USA}
\acmBooktitle{ACM Symposium on Cloud Computing (SoCC '24), November 20--22, 2024, Redmond, WA, USA}
\acmDOI{10.1145/3698038.3698553}
\acmISBN{979-8-4007-1286-9/24/11}
\begin{document}

\author{Amey Agrawal}
\affiliation{%
  \institution{Georgia Institute of Technology}
}
\email{ameyagrawal@gatech.edu}

\author{Sameer Reddy}
\authornote{Work done while at Georgia Institute of Technology.}
\affiliation{%
  \institution{Cisco Inc.}
}
\email{sameredd@cisco.com}

\author{Satwik Bhattamishra}
\affiliation{%
  \institution{University of Oxford}
}
\email{satwik.bmishra@cs.ox.ac.uk}

\author{Venkata Prabhakara Sarath Nookala}
\affiliation{%
  \institution{Meta Inc.}
}
\authornotemark[1]
\email{sarathnookala@meta.com}

\author{Vidushi Vashishth}
\affiliation{%
  \institution{Google Inc.}
}
\authornotemark[1]
\email{vvashishth@google.com}

\author{Kexin Rong}
\affiliation{%
  \institution{Georgia Institute of Technology}
}
\email{krong@gatech.edu}

\author{Alexey Tumanov}
\affiliation{%
  \institution{Georgia Institute of Technology}
}
\email{atumanov@gatech.edu}

\begin{abstract}
The likelihood of encountering in-training failures rises substantially with larger Deep Learning (DL) training workloads, leading to lost work and resource wastage. Such failures are typically offset by checkpointing, which comes at the cost of storage and network bandwidth overhead. State-of-the-art approaches involve lossy model compression mechanisms, which induce a tradeoff between the resulting model quality and compression ratio.
We make a key enabling observation that the sensitivity of model weights to compression varies during training, and different weights benefit from different quantization levels, ranging from retaining full precision to pruning. 
We propose (1) a non-uniform quantization scheme that leverages this variation, 
(2) an efficient search mechanism that dynamically finds the best quantization configurations, and
(3) a quantization-aware delta compression mechanism that rearranges weights to minimize checkpoint differences and thereby improving compression.

We instantiate these contributions in \sys, an in-training checkpoint compression system for DL workloads. Our experiments show that \sys consistently achieves a better tradeoff between accuracy and compression ratio compared to prior works, enabling a compression ratio up to 39x and withstanding up to 10 restores with negligible accuracy impact in fault-tolerant training. \sys achieves at least an order of magnitude reduction in checkpoint size for failure recovery and transfer learning without any loss of accuracy.

\end{abstract}

\maketitle
\renewcommand{\shortauthors}{Amey Agrawal et al.}

\section{Introduction}
\label{sec:introduction}

Large-scale Deep Learning (DL) training workloads, increasingly require weeks or months of compute time on GPU clusters~\cite{opt2}. As the duration and scale of these training efforts grows, so does the frequency of failures. For example, the DL training workloads in a Microsoft cluster encounter a failure every 45 minutes on average (excluding early failures) due to various system and user errors~\cite{philly}. This makes the role of checkpointing--creating periodic snapshots of a DL model during its training--increasingly crucial. When failure occurs, training can be recovered from the most recent checkpoint to reduce lost work.

During the model development phase, developers often resume from a stable mid-training checkpoint instead of starting from scratch for each bug fix or modification~\cite{opt-logbook}. Checkpoints are also used in transfer learning, where a saved model state can be adapted for\alexey{to?} a different task. Due to these varied uses, there is a growing research interest in checkpointing systems~\cite{QD, lcCompression, checknrun, deltadnn}.

Frequent checkpoints ensure minimal loss of progress in the event of failures but also increase the storage cost. For instance, checkpoints for language models like Pythia \cite{pythia}, OPT \cite{opt2} require tens of gigabytes of storage for each checkpoint. During the course of training, hundreds of such checkpoints are generated. As models grow in complexity and size, and as organizations seek to maintain a multitude of checkpoints, compressing these checkpoints becomes necessary.

\begin{figure}
    \centering
    \includegraphics[width=0.9\linewidth]{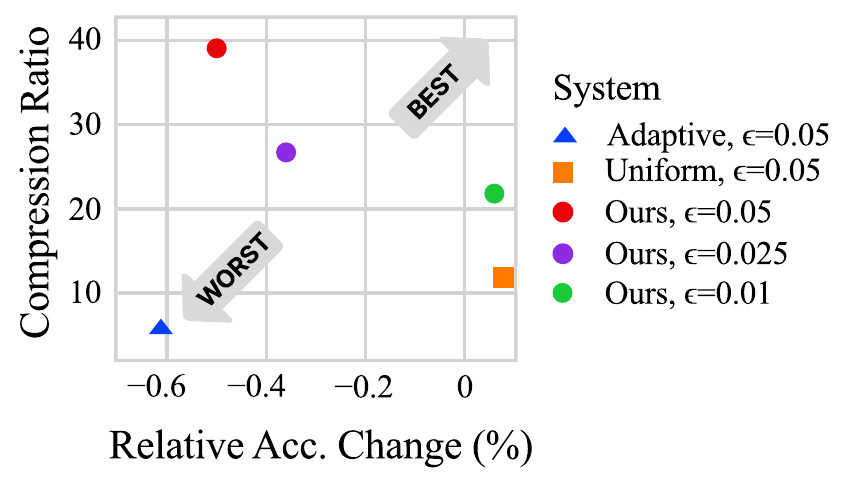}
    \caption{
        \small \sys provides better accuracy-storage tradeoffs compared to baselines for ResNet152 training.   }
    \label{fig:fr-tradeoff-single}
\end{figure}

As a result, several model compression systems have been proposed in the recent past \cite{QD,lcCompression,checknrun,deltadnn}, however, existing checkpoint compression systems face shortcomings in terms of suboptimal trade-off between compression ratio and accuracy degradation as shown in \cref{fig:fr-tradeoff-single}. We find that this issue originates due to the following factors. First, model parameters exhibit different levels of sensitivity to compression. Systems such as Check-N-Run \cite{checknrun} and QD-Compressor \cite{QD} adopt uniform quantization strategies that provide all parameters with the same level of resolution -- oblivious to the level of their impact on the model quality. 
Second, it's observed that models, as they accrue knowledge over the course of training, exhibit higher quantization error\alexey{either we observe or cite who observes}. However, existing systems adhere to fixed quantization configurations, e.g., a preset number of quantization bins, throughout training.

\begin{figure*}[t!]
    \centering
    \includegraphics[width=\textwidth]{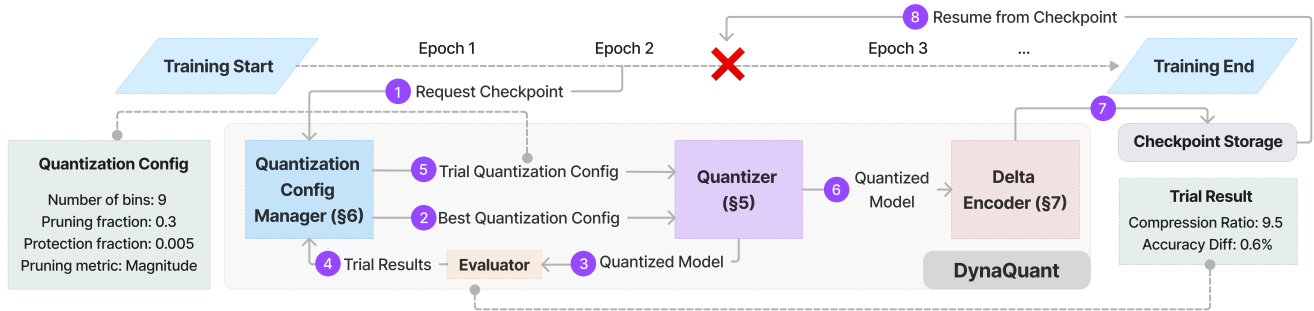}
    \caption{
        \small High-level workflow of a checkpoint compression and saving request, and resumption of training on failure from a saved compressed checkpoint in \sys.}
    \label{fig:hld}    
\end{figure*}

To overcome these limitations, we introduce \sys---an efficient, transparent in-training model checkpoint compression system for DL workloads (\autoref{fig:hld}). \sys is premised on the observation that not all model parameters contribute equally to model quality, and this contribution varies as the model trains\alexey{varies over the course of model training}. This observation informs three main aspects of \sys.

First, \sys uses a comprehensive quantization configuration space that partitions model parameters into three groups based on their significance: parameters preserved with full precision (unchanged), pruned parameters (set to zero), and quantized parameters. For the quantized parameters, \sys proposes a novel \textit{non-uniform} quantization approach using a sketch-based approximate K-Means clustering, achieving up to 65$
\times$ speed up compared to a na\"ive implementation. This enables more effective compression by reducing the number of required quantization bins by 2-3$\times$ compared to uniform quantization.

Second, \sys includes an efficient dynamic quantization configuration search component. This component automatically adapts the quantization configuration as parameters' sensitivity towards\alexey{to} compression changes over training time. This has the combined benefit of improving compression performance and alleviating the cognitive burden on ML practitioners to manually choose these configurations.

Finally, \sys implements a novel quantization-aware delta encoding scheme, based on the observation that most parameters remain in the same quantization bin across consecutive checkpoints.
Our delta encoding scheme rearranges the model parameters such that parameters for each quantization bucket are stored separately, improving delta encoding efficiency with compression ratios 3-4$\times$ higher than existing delta compression schemes.

\sys builds on the following key contributions:

\begin{itemize}
\item A novel \textit{non-uniform} quantization algorithm using approximate K-Means clustering. 
\item A mechanism for efficient, automatic quantization configuration search space navigation during training.
\item A delta compression algorithm with effective run length encoding enabled by quantization-aware model parameter rearrangement.

\end{itemize}

We evaluate \sys on 7 different model families, including tasks in vision and language modeling.
\sys reduces checkpoint size by 26-39$\times$ for fault-tolerant training and sustains up to ten failures (for multi-day training jobs) with negligible impact on the final accuracy of the trained model, achieving a 1.3-3.3$\times$ improvement over state-of-the-art methods. \sys can also reduce the storage overhead of snapshots of pre-trained models used for transfer learning by 10$\times$ with no impact on the performance of the fine-tuned model. Overall, \sys can achieve at least an order of magnitude checkpoint overhead reduction on both use cases with minimal accuracy loss. 

\section{Background and Use Cases}
\label{sec:background}

We start by introducing the notion of checkpointing in deep learning systems and highlight a variety of scenarios where model checkpoint compression is beneficial. Model checkpoints serve as snapshots of a deep learning model at a certain stage in its training process. They encompass the model's architecture, learned parameters, and optimizer state. One can think of checkpoints as analogous to code commits in version control systems, where each checkpoint represents a version of the model at a specific stage of its development. In the rest of this section, we describe several use cases for model checkpoint compression.

\minihead{Failure Recovery In Training Workloads} 
Large-scale deep-learning training workloads are highly susceptible to various hardware and software failures. 
Any failure during this period prompts recovery from the most recent checkpoint. 
Joen et al. \cite{philly} show that on average DL training workloads in a Microsoft cluster encounter a failure every 45 minutes (excluding early failures) due to various system and user errors. Zhang et al. \cite{opt2} also observe that the training of the OPT-2 
model required 100+ restarts due to failures across the span
of two months. As we move toward developing not only larger but also more complex models and distributed training systems, the likelihood of experiencing system failures is expected to rise. Simultaneously, constraints on bandwidth and storage capacity, especially in shared multi-tenant environments, limit the frequency of checkpoints. This creates tension between frequent checkpointing to minimize wasted work and managing the storage and network checkpointing overheads. Addressing this challenge calls for a compression mechanism that enables more frequent checkpoints with minimal overhead.

\minihead{Iterative Model Development} In the development of deep learning models, checkpoints  are useful for multiple use cases including version control and continual learning. DL model development is an iterative process with bug fixes, hyperparameter tuning, and architectural adjustments. Given the substantial GPU hours invested in training, it is wasteful to discard progress just to make these adjustments. As a result, practitioners make these changes mid-way through training and resume the training from a stable checkpoint~\cite{opt-logbook}. This prompts the need for an effective git-like version control system for models, to manage their development cycle. Beyond the training phase, production environments often require continuous updates to models to accommodate new data and maintain performance. In such scenarios, model checkpoints are preserved for extended periods for reasons such as debugging, reliability, and provenance. Delta-encoded checkpoints, where only the changes between successive checkpoints are stored, can help minimize storage requirements and streamline recovery processes in these use cases.

\minihead{Model Hubs \& Transfer Learning} Model hubs, such as Hugging Face~\cite{huggingface}, provide pre-trained models for a variety of tasks. These pre-trained models can be used for transfer learning, a  process in which a pre-trained model is fine-tuned for a new task or domain. With the growth of the number of models on such hubs, the bandwidth required to transfer model checkpoints becomes a significant concern. For example, the BERT Base \cite{bert} model alone was downloaded more than 41 million times from the HuggingFace model hub \cite{hf-bert} in a 30-day window as of August 2023. Model checkpoint compression can reduce the bandwidth required for transferring models between users and model hubs.

\section{Overview}
\label{sec:background}

This section outlines the design goals and key components of \sys.

\begin{table}[t]
    \centering
    \begin{tabular}{@{\hspace{4px}}l*{6}{@{\hspace{4px}}c}@{\hspace{4px}}}
    \toprule
     & \rot{CNR~[\citenum{checknrun}]} & \rot{QD~[\citenum{QD}]} & \rot{LC~[\citenum{lcCompression}]} & \rot{DDNN~[\citenum{deltadnn}]} & \rot{GOBO~[\citenum{gobo}]} & \rot{\sys} \\
    \midrule
    In-training & \greencheck & \greencheck & \greencheck & \greencheck & \redcross & \greencheck \\
    Low Overhead & \greencheck & \greencheck & \greencheck & \greencheck & \redcross & \greencheck \\
    Model Agnostic & \redcross & \greencheck & \greencheck & \greencheck & \greencheck & \greencheck \\
    Scalable & \greencheck & \greencheck & \greencheck & \greencheck & \redcross & \greencheck \\
    Algorithmically Transparent & \greencheck & \redcross & \greencheck & \greencheck & \greencheck & \greencheck \\
    \midrule
    Dynamic Configuration Management & \redcross & \redcross & \redcross & \greencheck & \redcross & \greencheck \\
    Non-uniform Quantization & \redcross & \redcross & \greencheck & \redcross & \greencheck & \greencheck \\
    Quantization-aware Delta Encoding & \redcross & \greencheck & \greencheck & \greencheck & \redcross & \greencheck \\
    \bottomrule
    \end{tabular}
    \vspace{0.5em}
    \caption{
        \small Comparison of checkpoint compression systems.
    }
    \label{tab:system-featurs}
    \end{table}

\minihead{Design Goals}
\sys aims to preserve model accuracy under multiple restores from compressed checkpoints while minimizing storage overhead. Additional design goals include:

\begin{itemize}
\item Low Overhead: Compression should impose minimal overhead ($<0.5$\%) on training runtime.
\item Model Agnostic: The quantization algorithm should handle diverse model architectures without requiring model-specific adjustments.
\item Scalable: Ability to process models with hundreds of millions to billions of parameters efficiently.
\item Algorithmically Transparent: No interference with the original training algorithm, unlike quantization-aware training methods.
\end{itemize}

\autoref{tab:system-featurs} compares the main features of \sys with existing checkpoint compression systems. A detailed discussion is provided in \S~\ref{sec:related}.

\minihead{System Overview}
\sys processes checkpoint requests in three stages, as illustrated in \autoref{fig:hld}. First, the \qcm performs a distributed search to identify an optimal quantization configuration. Next, the \quantizer compresses the model using the identified configuration. Finally, the \deltaE processes the compressed checkpoint in CPU memory.

\noindent{\textbf{\qcm} (\S~\ref{sec:search}):} 
The quantization configuration (\autoref{fig:hld}, \autoref{tab:dq-search-space}) summarizes all quantization related parameters. Over the course of training, \sys dynamically updates the quantization configuration throughout training to maximize the compression while satisfying to user-defined  quality threshold ($\epsilon$) for maximum allowed quantization error per checkpoint. An efficient search algorithm identifies a configuration maximizing compression ratio while adhering to $\epsilon$.

\noindent{\textbf{\quantizer} (\S~\ref{sec:quantization}):} 
Given a specific quantization configuration, \sys divides all the model parameters into three categories accordingly: 
prune (set to zero if below the pruning threshold), protect (preserve in full precision if above the protection threshold), or quantize (apply non-uniform quantization).
\sys uses a novel efficient non-uniform quantization scheme using sensitivity-aware approximate K-Means clustering, where quantization granularity is determined by the number of bins used in K-Means (e.g., 8-bins=3-bits). This approach offers superior quantized models with reduced runtime overhead compared to traditional methods and allows for any number of quantization bins.

\noindent{\textbf{\deltaE} (\S~\ref{sec:delta}):} 
\sys performs lossless compression using quantization-aware delta compression, exploiting similarities between consecutive checkpoints. \deltaE employs a novel parameter rearrangement technique enabling efficient run-length encoding, reducing storage overhead by up to two orders of magnitude.
After delta-compression and run-length encoding, model parameters are encoded in a combined byte stream format.

\begin{figure*}[ht!]
    \centering
    \includegraphics[width=\textwidth]{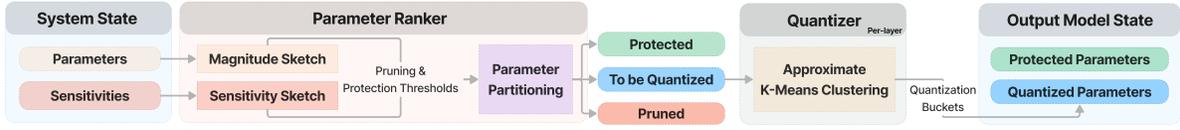}
    \caption{
    \small 
        A high-level illustration of the quantization process in \sys.
    }
    \label{fig:quantizer}    
\end{figure*}

\section{\quantizer}
\label{sec:quantization}

This section outlines the design and implementation of the \quantizer. An overview of \quantizer's workflow is shown in \autoref{fig:quantizer}. First, the \quantizer ranks model parameters according to their magnitude and sensitivity (\S~\ref{sec:ranking}). Based on the ranking, the parameters are then partitioned into three groups (\S~\ref{sec:pruneandprotect}). The least important parameters are pruned, the most important ones are protected with high precision, while the remaining parameters are quantized using our novel non-uniform quantization approach (\S~\ref{sec:nonuniform}).

\subsection{Metrics for Parameter Ranking}
\label{sec:ranking}
First, we need to rank parameters by importance in order to assign them for protection, pruning, or quantization.

The two popular metrics used to determine the importance of parameters are magnitude and sensitivity. Magnitude-based ranking approaches simply use the parameter magnitude as an importance score~\cite{han2015learning}. The magnitude score $I_m(w)$ for parameter $w$ is defined as $I_m(w) = |w|$. Sensitivity-based ranking approaches attempt to estimate the impact on the end-to-end loss value when a parameter is altered \cite{lecun1989optimal,dong2019hawq}. In this paper, we use the first-order Taylor series approximation formulation of sensitivity for the sake of computational efficiency. 
The sensitivity score is defined as $I_s(w) = |\nabla\mathcal{L}(w) w|$, where $\mathcal{L}$ represents the loss function and $\nabla\mathcal{L}(w)$ is the gradient of weight $w$ with respect to the loss function. It has been shown that magnitude-based parameter ranking approaches perform poorly in the early stage of training~\cite{when-to-prune}. On the other hand, as we approach convergence, gradients start to diminish, making the sensitivity metric less reliable.

To address the above limitations, we propose to use a combination of sensitivity and magnitude to get a reliable ranking throughout the training process. \cref{fig:param-dist} shows the distribution of parameters across the two ranking dimensions for a ResNet152 model trained on Imagenet data. While there is a strong correlation between magnitude and sensitivity, we observe a significant number of outliers, which have a high magnitude score but a low sensitivity score or vice versa. By ranking parameters across both the metrics, \sys can identify and preserve these parameters.

\begin{figure}
    \centering
    \begin{subfigure}{1\linewidth}
    \begin{subfigure}{0.49\linewidth}
        \centering
        \includegraphics[width=0.95\linewidth]{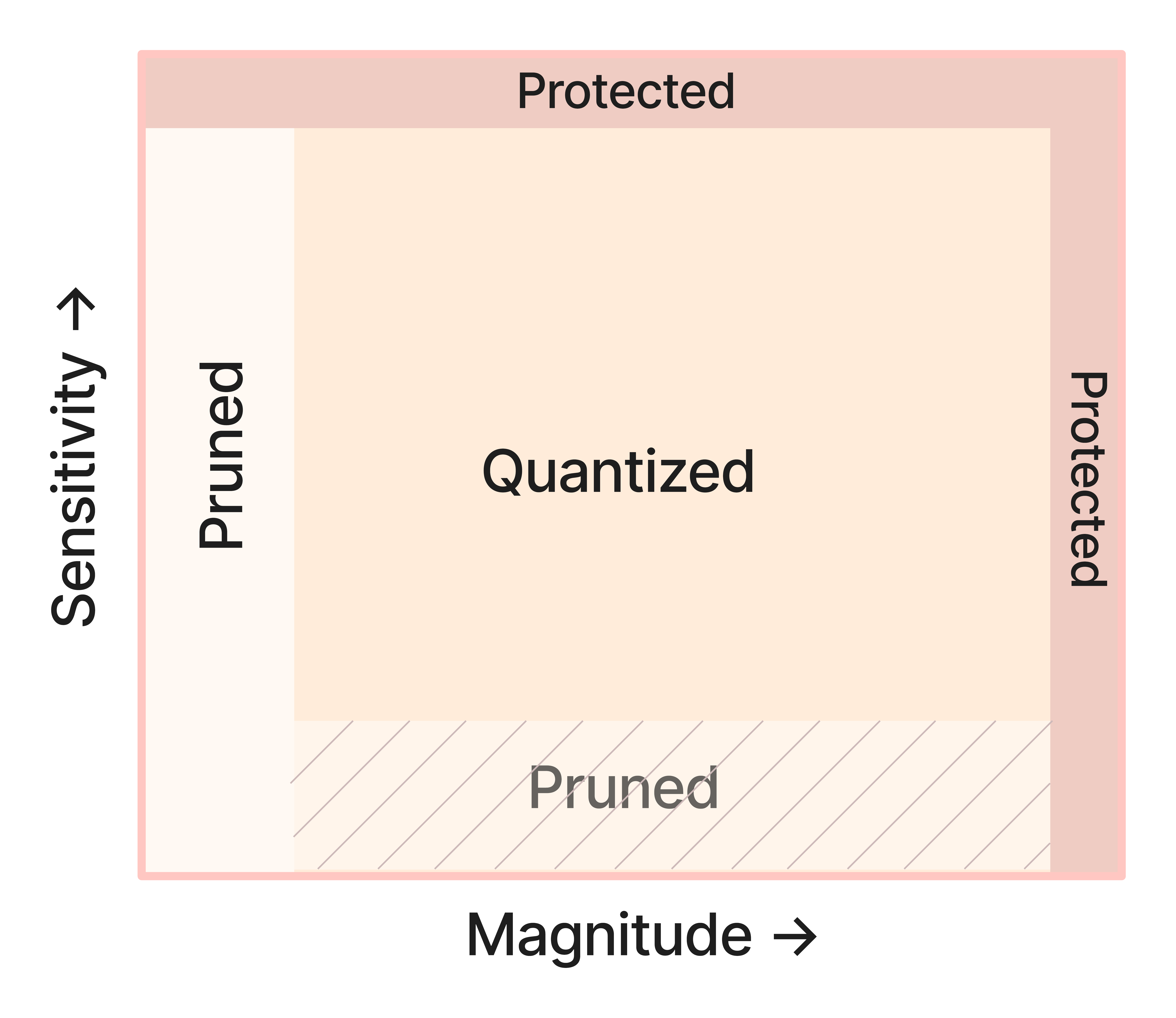}
        \caption{\small Parameter partitioning schema in \sys.}
        \label{fig:param-grid}
    \end{subfigure}
    \hfill
    \begin{subfigure}{0.49\linewidth}
        \centering
        \includegraphics[width=0.9\linewidth]{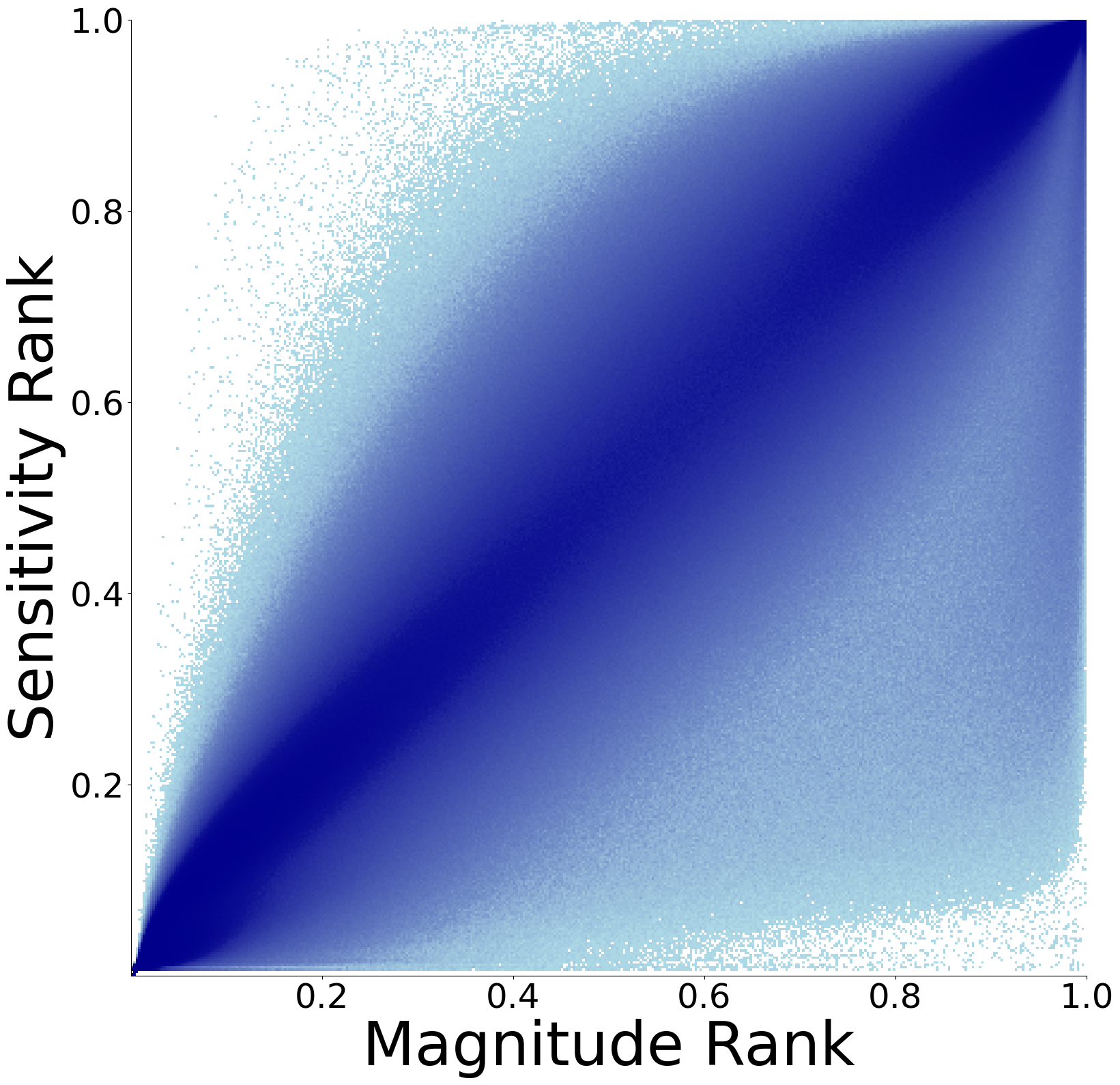}
        \caption{\small{Parameter distribution for ResNet152.}}
        \label{fig:param-dist}
    \end{subfigure}
    \end{subfigure}
    \caption{\small Model parameters are partitioned into three groups for pruning, protection, and quantization by jointly evaluating magnitude and sensitivity scores.}
    \vspace{-1em}
\end{figure}

For efficient sensitivity computation, we reuse gradients computed during the training process for sensitivity estimates and asynchronously copy the gradients to pinned CPU memory to avoid GPU memory overhead. We discuss implementation details in \cref{app:sensitivity-computation}.

\subsection{Parameter Partitioning}
\label{sec:pruneandprotect}

Given the magnitude and sensitivity score for each model parameter, the \quantizer partitions the parameters into one of the three groups (\cref{fig:param-grid}). 

\noindent\textbf{Protection.}
\sys preserves the most important model parameters in the high precision \texttt{bfloat16} format.
Recent quantization studies~\cite{llm-int8, dash2020hessian} show a significant reduction in quantization error when a few important weights are protected. Our experiments suggest that protecting a small fraction (0.05-0.1\%) of both the highest magnitude and sensitivity parameters provides the best performance. The precise protection fraction is determined dynamically by the \qcm. Based on the fraction ($F_{\text{prot}}$) of parameters to be protected, we identify the minimum magnitude and sensitivity thresholds (quantiles) for a parameter to qualify for protection.

\noindent\textbf{Pruning.} We prune the least significant model parameters, based on the pruning fraction $F_{\text{prun}}$ (typically between 10-40\%) and the pruning metric, magnitude or sensitivity. An important design consideration is the granularity level at which the pruning is performed. Layer-level pruning can degrade model quality due to different redundancy levels across the network layers.  Conversely, global pruning may disproportionately affect different layer types (e.g., Convolution, Attention, Linear) due to varied weight distributions. To tackle this, we perform {\em per-layer type} pruning, i.e. using the same pruning fraction across all layers of a given type. We observed that in GPT-2 Medium, we can eliminate up to 30\% parameters with a minimal effect on the model's quality.

\subsection{Efficient Parameter Partitioning via Quantile Sketches} 
\label{app:quantile}

Computing quantiles required for determining pruning and protection thresholds can be costly, especially for large models due to the log-linear time complexity of the operation. Moreover, in cases where models are trained via model parallelism, and parameters are distributed across different worker groups, computing global quantiles can be quite challenging. %

To facilitate efficient partitioning of model parameters, we use a sketch-based quantile estimation algorithm that is a simplified and parallelized version of DDSketch~\cite{ddsketch}. This enables us to compute approximate quantile estimates in linear time with a configurable $\alpha$-relative-error bound, that is, the sketch produces an estimated quantile value corresponding to the quantile value $x_q$ such that $|\tilde{x_q} - x_q| \leq \alpha x_q$. For computing protection and pruning thresholds, that require tail quantiles, DDSketch's relative error bounds outperform alternative quantile sketches that guarantee rank error. Furthermore, the quantile sketches are mergeable, which simplifies the computation of global quantile estimates from individual model shards.

The sketching algorithm transforms the input values into a logarithmic space and then divides them into an equal-width histogram (\cref{fig:log-sketch-transform}). Due to the use of a logarithmic scale, we achieve varying levels of detail for different ranges of the original values. Specifically, smaller original values will land in narrower bins, while larger values will fall into broader bins. This allows the sketch to provide the $\alpha$-relative error property of the sketch. To produce quantile estimates, the sketch sums up the buckets until it finds the bucket containing the desired quantile value.

Algorithm \ref{algo:quantile-sketch} describes the parallelizable version of DDSketch \cite{ddsketch} optimized for GPU execution. We forgo the bucket merging step in the original DDSketch for ease of parallelization. This has limited impact in practice, as we observe that the number of buckets does not exceed a few thousand even for models with hundreds of millions of parameters with a strict relative error bound of 1\%. We implement a highly parallelized version of this algorithm using Pytorch that can run on GPUs. We observe a 3-4$\times$ speed-up compared to off-the-shelf GPU-enabled implementation provided by CuPy \cite{cupy} while using significantly less memory.

\begin{figure}
    \centering
    \includegraphics[width=\linewidth]{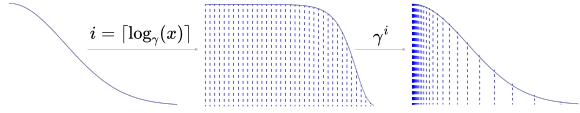}
    \vspace{-1em}
    \caption{
        \small{
        {DDSketch's log space transformation creates a histogram with value-dependent bucket width.}
        }
    }
    \label{fig:log-sketch-transform}
\end{figure}

\begin{algorithm}
\caption{GetQuantileSketchHistogram}
\begin{algorithmic}[1]
\REQUIRE $X$ \COMMENT{Input list}
\REQUIRE $\alpha$ \COMMENT{Relative Error Bound}
\STATE $\gamma \leftarrow (1 + \alpha) / (1 - \alpha)$ \\
\STATE Define $X_q$ as an empty array \\
\FORALL{$x$ in $X$}
    \STATE Append $\lceil \log_{\gamma}(x) \rceil$ to $X_q$
\ENDFOR
\STATE $(\texttt{bucket\_vals}, \texttt{bucket\_counts}) \leftarrow \texttt{unique}(X_q)$
\RETURN $\texttt{bucket\_vals}, \texttt{bucket\_counts}$
\end{algorithmic}\label{algo:quantile-sketch}
\end{algorithm}

\subsection{Non-uniform Quantization via Approximate K-Means Clustering}
\label{sec:nonuniform}
The remaining model parameters that are not pruned or protected go through quantization. 
Quantization reduces the precision of model parameters to lower-bit representations (e.g., from 32-bit floating-point to 8-bit integers), leveraging redundancy and noise tolerance in large networks. Uniform quantization strategies create equally spaced quantization levels across the parameter space, while non-uniform quantization approaches use varying intervals between levels to adapt to the parameter distribution. 

Non-uniform quantization techniques, such as via \kmeans clustering~\cite{checknrun, gobo, deep-compression}, generally provide better compression and model quality. However, they have seen limited adoption in checkpoint compression systems due to their computational complexity. For instance, GOBO~\cite{gobo}, that utilizes a variant of \kmeans clustering, takes roughly 50 minutes to quantize a small model like ResNet-18 (11.7M parameters).

\begin{algorithm}
\caption{Approximate K-Means Clustering}
\begin{algorithmic}[1]
\REQUIRE $X$ \COMMENT{Input data points}, $k$ \COMMENT{Number of clusters}
\REQUIRE $\alpha$ \COMMENT{Relative error}, $\sigma$ \COMMENT{Linear combination factor}
\STATE $X_q$, $C_q$ $\leftarrow$ \texttt{GetQuantileSketchHistogram}($X$, $\alpha$) \COMMENT{Obtain log sketch histogram \cref{algo:quantile-sketch} for coarse grain clustering. $X_q$ are the histogram bucket centers and $C_q$ are the corresponding frequencies.}
\STATE  $\hat{X_q}$, $\hat{C_q}$ $\leftarrow$  \texttt{Normalize}($X_q$), \texttt{Normalize}($C_q$)
\STATE $W = \sigma * \hat{C_q} + (1 - \sigma) * |\hat{X_q}|$  \COMMENT{Compute sample weight for histogram buckets according to frequency and magnitude} 
\STATE $\texttt{centers}$ $\leftarrow$ \texttt{WeightedKMeans}($X_q$, $W$, $k$)  \COMMENT{Perform weighted k-means++ on histogram keys and values}  
\RETURN $\texttt{centers}$
\end{algorithmic}\label{algo:kmeans}
\end{algorithm}

\minihead{Approximate \kmeans Clustering}
We introduce a novel non-uniform quantization method leveraging approximate K-means clustering, which strikes a balance between quantization quality and computational efficiency. Our key insight is to perform K-means clustering on histogram buckets of the parameters, rather than on the raw parameters themselves. This significantly reduces the computational burden of K-means clustering, resulting in up to 65$\times$ speedup compared to the standard implementation in CuML (\S~\ref{sec:ablations}).

\cref{algo:kmeans} describes the two-step quantization algorithm. First, we group model parameters into coarse-grain buckets, utilizing the same log-space projection approach in DDSketch. We then perform weighted k-means clustering, where the histogram bins serve as the data and their heights as sample weights, to establish quantization boundaries. Two key optimizations enhance this clustering performance:

\begin{figure}
    \centering
    \includegraphics[width=\linewidth]{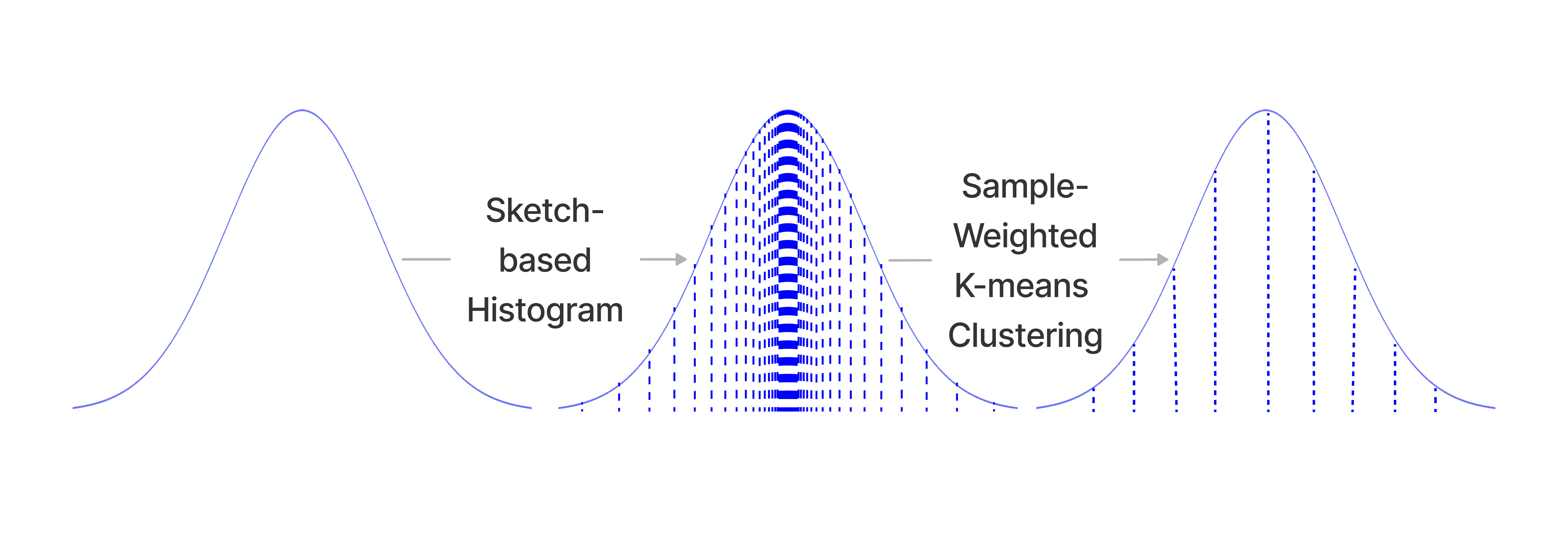}
        \vspace{-1em}
    \caption{
        \small{
        {Two-step quantization: 
        (1) group model parameters into coarse-grained buckets with sketch-based histogram computation. (2) cluster parameter buckets using sample weighted K-means clustering.}
        }}
    \label{fig:two-step-quantization}
\end{figure}

\begin{algorithm}
\caption{Weighted K-Means++ Initialization}
\begin{algorithmic}[1]
\REQUIRE $X$ \COMMENT{Input data points}
\REQUIRE $W$ \COMMENT{Input sample weights}
\REQUIRE $k$ \COMMENT{Number of clusters}
\STATE $C \leftarrow X[ \texttt{random\_choice}(W)]$ \COMMENT{Initialize the first centroid randomly according to weights}
\FOR{$i$ in $2$ to $k$}
    \STATE $D \leftarrow \texttt{min\_distance}(X, C)$ \COMMENT{Compute distance of each point to nearest centroid}
    \STATE $D \leftarrow D \cdot W$ \COMMENT{Weight the distances}
    \STATE $P \leftarrow D / \texttt{sum}(D)$ \COMMENT{Compute probabilities}
    \STATE $C \leftarrow C \cup X[ \texttt{random\_choice}(P)]$ \COMMENT{Select new centroid}
\ENDFOR
\RETURN $C$
\end{algorithmic}\label{algo:initialize}
\end{algorithm}

First, we introduce an additional parameter, $\sigma$, to our weighted \kmeans to help balance the allocation of resolution between high-importance parameters with low frequency and high-frequency parameters with low importance. Usually, non-uniform clustering offers increased resolution to denser parameter spaces. However, for neural network parameters, this would allocate greater resolution to values near zero. To address this, we calculate the sample weight of a bucket, $b^{i}$, as a linear combination of the bucket's normalized frequency and magnitude: $w^{i} = \sigma * b^{i}_{\text{freq}} + (1 - \sigma) * b^{i}_{\text{mag}}$ (see line 2 in \cref{algo:kmeans}). Our observations suggest that values of $\sigma$ in the range $0.1-0.4$ generally provide the best compression results. We use $\sigma=0.2$ for all experiments.

Second, the success of the \kmeans clustering algorithm relies heavily on proper initialization. Inappropriate initialization can lead to suboptimal solutions and slower convergence. K-means++~\cite{kmeanspp} addresses this by ensuring the initial centroids are uniformly distributed across the space, often leading to near-optimal solutions with a single round of clustering. We extend the K-means++ \cite{kmeanspp} algorithm for initialization with support for sample weights, \Cref{algo:initialize} provides a sketch of this process.

\minihead{Error Analysis} In the first step of the quantization algorithm, we create a histogram using DDSketch's log-space projection approach, which provides a configurable $\alpha$-relative error bound. Using this approach, we formally show that the difference between optimal loss of \kmeans clustering over raw inputs $X$ and its histogram counterpart $\tX$ is bounded by $\alpha$. We can further show that when the clustering is performed with sample-weighted k-means++, then the expected loss over $\tX$ is bounded by the optimal loss over $X$ and an additional term related to $\alpha^2$ (\cref{app:kmeans}).

\subsection{Efficient Sensitivity Computation} 
\label{app:sensitivity-computation}
To save computational resources, we reuse gradients computed during the training process for sensitivity estimates instead of calculating them separately. To avoid any GPU memory overhead, we asynchronously copy the gradients to pinned CPU memory as soon as the backward pass is completed. The gradient copy operation is overlapped with the optimizer step, data loading, and forward pass of the next batch, making the overhead of this operation negligible as shown in \cref{fig:gradient-copy-timeline}. The copied gradient values are accumulated on the CPU using an exponential moving average $\text{EMA}^{t}_{g_w} = \beta * g^{t}_w + (1 - \beta) * \text{EMA}^{t - 1}_{g_w}$, where $\text{EMA}^{t}_{g_w}$ is the exponential moving average of the gradient $g_w$ at timestep $t$. We set the value of $\beta$ to be $0.9$ in all our experiments to emphasize the recent gradients. The gradient copy mechanism is only invoked for a limited preconfigured number of batches before checkpointing is scheduled to be performed to further avoid any impact on steady-state throughput. We identify that typically recording gradients for fifty batches is sufficient to get good sensitivity estimates. 

\begin{figure}[h]
    \centering
    \includegraphics[width=\linewidth]{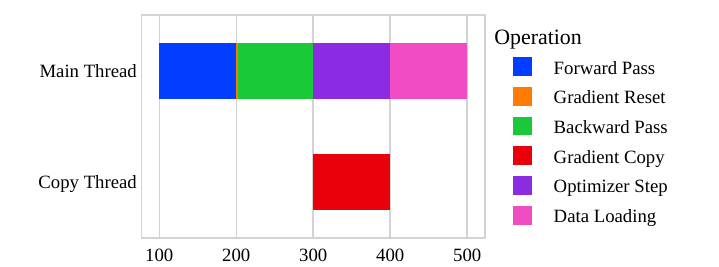}
    \caption{\small{{Schematic representation of asynchronous gradient offloading for efficient sensitivity computation in \sys.}}}
    \label{fig:gradient-copy-timeline} 
\end{figure}

\section{Quantization Config Manager}
\label{sec:search}

In this section, we describe \qcm's search space of quantization configurations and its efficient search algorithm. Unlike checkpoint compression systems that use a predefined quantization configuration throughout training, the manager dynamically searches for a configuration that maximizes compression ratio while satisfying the quality constraint $\epsilon$ at every checkpoint.

\noindent\textbf{Search Space.} \autoref{tab:dq-search-space} describes the quantization configurations that comprise the search space of \qcm. To keep the search space manageable, we choose to use the same number of quantization bins across all layers except embedding layers. We use separate quantization parameters for embedding layers, since we have observed that embedding layers are considerably more sensitive to quantization, in line with findings from other research on Transformer model quantization \cite{gobo, Qbert}.

\begin{table}[h]
\centering
\begin{tabular}{@{\hspace{0px}} l l @{\hspace{0px}}}\\
\toprule
Parameter & Values \\
\midrule
Number of bins & 4, 6, 8, 12, 16, 32 \\
Pruning Fraction & 0, 0.1, 0.2, 0.3, 0.4, 0.5 \\
Pruning Metric & Magnitude, Sensitivity \\
Protection Fraction & 0.0005, 0.005, 0.01 \\
\bottomrule
\end{tabular}
\vspace{0.5em}
\caption{\small \sys's quantization configuration space.}
\label{tab:dq-search-space}
\end{table}

\noindent\textbf{Overview of the Search Algorithm.} Na\"ively evaluating each configuration in the search space is costly -- the search space contains over 200 configurations for non-embedding layers alone. To speed up the search, our key insight is that the optimal quantization configuration remains largely similar between adjacent checkpoints. At the first checkpoint, we perform a guided exhaustive search to identify the optimal configuration. Subsequently, we greedily evaluate the configurations that are close to the previous best configuration and only fall back to the exhaustive search if we can not identify a configuration that satisfies the quality constraint.  We observe that we need to resort to exhaustive search at most 2-3 times during model training.

\begin{algorithm}[t]
\small
\caption{Guided Exhaustive Search}
\begin{algorithmic}[1]
\REQUIRE Quality threshold $T$, Configuration cube $C_{\text{cube}}$, Compression Ratio $CR_{\text{max}}$
\ENSURE Optimal configuration $C_{\text{opt}}$

\STATE $C_{\text{opt}} \gets \text{Null}$
\STATE $CR_{\text{max}} \gets 0$

\STATE $Diagonal \gets \text{GetDiagonal}(C_{\text{cube}})$
\STATE $C_{\text{opt}}, CR_{\text{max}} \gets \text{ParallelSearch}(Diagonal, T)$
\STATE $SubCubes \gets \text{GetFeasibleSubCubes}(C_{\text{cube}}, C_{\text{opt}})$
\FOR {each $SubCube$ in $SubCubes$}
\STATE $C_{\text{temp}}, CR_{\text{temp}} \gets \text{GuidedExhaustiveSearch}(T, SubCube)$
\IF{$CR_{\text{temp}} > CR_{\text{max}}$}
\STATE $C_{\text{opt}} \gets C_{\text{temp}}$
\STATE $CR_{\text{max}} \gets CR_{\text{temp}}$
\ENDIF
\ENDFOR
\RETURN $C_{\text{opt}}$
\end{algorithmic}\label{algo:guided-search}
\end{algorithm}

\noindent\textbf{Guided Exhaustive Search.} We use domain knowledge to reduce the cost of exhaustive search. If there was only one parameter, like the number of quantization bins, model quality would increase monotonically with the number of bins, while the compression ratio decreases. In this setting, we can find the optimal number of bins using a variation of binary search. We apply this approach to multiple configuration parameters, organizing them in a `configuration cube' where each axis represents a configuration knob as shown in \cref{fig:config-cube}. This layout ensures a monotonic increase in model quality across each axis, allowing us to use a divide-and-conquer method (\cref{algo:guided-search}). We repeat this search procedure with the two pruning metrics, magnitude and sensitivity, and select the better of the two.

\begin{figure}
    \centering
    \includegraphics[width=0.75\linewidth]{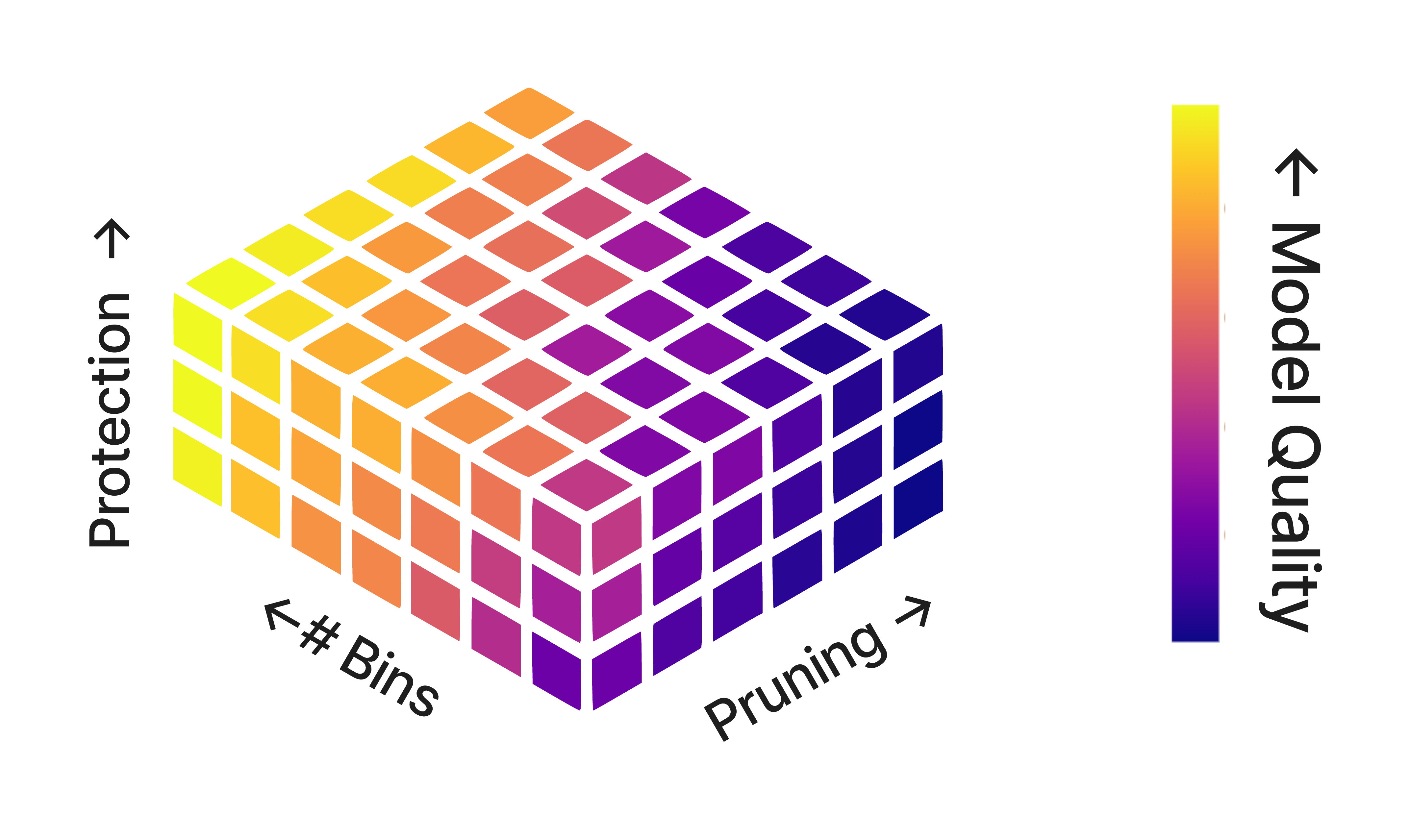}
    \caption{\small 
        Configuration search space as a function of pruning fraction, protection fraction, and number of bins. The color gradient indicates the direction of higher quality.%
    }
    \vspace{-1em}
    \label{fig:config-cube}
\end{figure}

\noindent\textbf{Delta-Neighbourhood Search.} In the steady-state, we obtain a suitable configuration by greedily selecting the best configuration within the $e$-delta neighborhood of the configuration chosen at the previous checkpoint. The $e$-delta neighborhood is defined as all configurations that are at most $e$ steps away in the configuration cube from the given configuration. At the start of each neighborhood search, we evaluate a configuration identical to the last best configuration, except with a different pruning metric. If we observe a significant improvement in the quantized model’s quality, we change the pruning metric and proceed with the neighborhood search using the new metric. As training progresses, the model becomes denser and more sensitive to compression. Hence, subsequent neighborhood searches intentionally rule out configurations more aggressive than the prior setting.

\noindent\textbf{Parallelized Implementation.} Since each data parallel replica maintains an identical model state, we can parallelize the search process. Configuration search trials are evaluated in groups of $m$, where $m$ is the degree of data parallelism. While scheduling the trials, we arrange them in a way that allows for an early exit. For instance, in the delta neighborhood search, we order the configurations in decreasing order of compression ratio. If in a given round of evaluation, we find a solution that satisfies the quality constraint, we can exit the search without evaluating the rest of the configurations. Our experiments show that for a data parallelism degree of eight, a high-quality configuration can be found within a single round of evaluation. %

\section{\deltaE}
\label{sec:delta}

To enhance the storage efficiency of compressed models, we introduce a \deltaE that stores the differences between successive checkpoints. Delta encoding works well with quantized models, tracking changes between quantization buckets rather than exact values.

\noindent\textbf{Calculating Delta.} To compute the delta between the adjacent checkpoints, we transform quantized values into integers and calculate the difference between the quantized integer values of the successive checkpoints for each layer.
We employ the technique proposed in QD-compressor~\cite{QD} that treats the quantization range as a cyclic buffer. Concretely, for integer-mapped quantized checkpoints $C^{i-1}_q$ and $C^i_q$, we compute the delta $D_i$ as $ D_i = (C^{i-1}_q - C^i_q) \text{ mod } B$, where $B$ is the number of quantization bins. %

Due to the dynamic changes in the system configuration, the number of quantization bins may vary between checkpoints. We address this by adjusting the size of the cyclic buffer to the maximum number of bins in either checkpoint. Additionally, pruned and protected values in our quantization scheme are treated as unique quantization levels with the same delta calculation. The protected values from checkpoint $C^i$ are stored separately to aid reconstruction.

\noindent\textbf{Encoding Delta.}
After delta calculation, we use a combination of run-length encoding \cite{rle} and Huffman \cite{huffman} encoding to reduce its storage footprint. Given that parameter changes are minimal in the later training stages, run-length encoding is particularly effective for compression. Standard run-length encoding can incur significant overhead due to the values with a run-length of one. We can reduce this overhead by only storing the run-lengths only when it is greater than one. However, this requires a mechanism to distinguish between the encoded values and their run-lengths. We achieve this by simply negating the sign of all the values, such that all the run-lengths are positive while all delta values are less than or equal to zero. Huffman encoding is then applied to the run-length encoded data for further compression.

\begin{figure}
    \centering
    \includegraphics[width=0.95\linewidth]{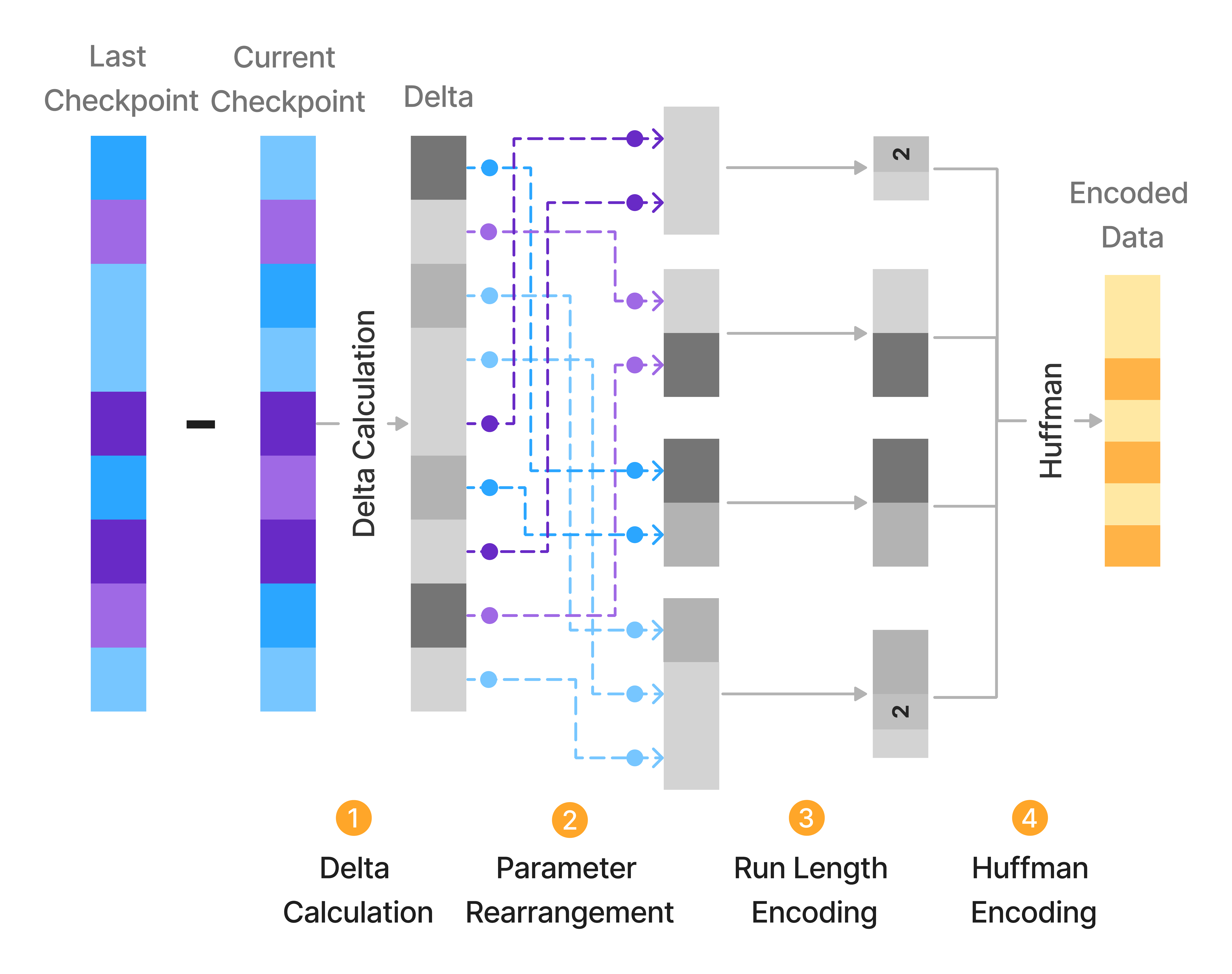}
    \caption{\small{{Rearrange and split computed delta based on the quantization bucket each parameter belonged to in the last checkpoint.%
    }}}
    \vspace{-1.5em}
    \label{fig:delta}
\end{figure}

\noindent\textbf{Optimization: Parameter Rearrangement.}
\label{sec:param-rearrangement}
The efficiency of run-length encoding depends on the arrangement of the parameters. We observe that the \emph{migration rate} (i.e., the fraction of parameters that change the quantization bucket between adjacent checkpoints) of parameters can differ significantly between different quantization buckets. In the default arrangement where parameters of all quantization buckets are interleaved, the run lengths get interrupted by the buckets with high migration rates, resulting in poor compression.

To prevent this, we rearrange parameters to isolate the parameters from different quantization buckets. After calculating the delta, we split delta values corresponding to each quantization bucket in the source checkpoint into separate groups, as shown in the second step of \cref{fig:delta}. Each group is then encoded separately. This ensures that quantization bins with high migration rates don't affect others. The process can easily be reversed during reconstruction and does not require any additional storage to obtain the parameter arrangement. \cref{algo:param-rearrangement,algo:param-reconstruction} provide a sketch of this algorithm. The approach is particularly useful when the number of quantization bins changes between checkpoints. As demonstrated in \S~\ref{sec:ablations}, this optimization can significantly improve the compression ratio.

\begin{algorithm}[t]
\caption{Rearranging parameters for efficient encoding}
\begin{algorithmic}[1]
\REQUIRE $\Delta$ (Delta values calculated from adjacent checkpoints)
\REQUIRE $C^{i-1}_q$ (The quantized integer values of the last checkpoint)
\ENSURE Delta groups (A dictionary containing delta values for each unique quantization bin)
\STATE bins $\leftarrow$ unique\_bins($C^{i-1}_q$)
\STATE delta\_groups $\leftarrow$ create\_empty\_dict(bins)
\FOR{$j \in [0, \text{length}(C^{i-1}_q)]$}
\STATE delta\_groups[$C^{i-1}_q[j]$].append($\Delta[j]$)
\ENDFOR
\RETURN delta\_groups
\end{algorithmic}\label{algo:param-rearrangement}
\end{algorithm}

\begin{algorithm}
\caption{Reconstructing parameters after rearrangement}
\begin{algorithmic}[1]
\REQUIRE delta\_groups (Delta groups obtained from the rearranging step)
\REQUIRE $C^{i-1}_q$ (The quantized integer values of the previous checkpoint)
\ENSURE $C^i_q$ (The quantized integer values of the current checkpoint)
\STATE $C^i_q$ $\leftarrow$ create\_zero\_array(length($C^{i-1}_q$))
\FOR{$j \in [0, \text{length}(C^{i-1}_q)]$}
\STATE bin $\leftarrow$ $C^{i-1}_q[j]$
\IF{delta\_groups[bin] $\neq$ $\emptyset$}
\STATE $C^i_q[j]$ $\leftarrow$ $C^{i-1}_q[j]$ + delta\_groups[bin].pop(0)
\ENDIF
\ENDFOR
\RETURN $C^i_q$
\end{algorithmic}\label{algo:param-reconstruction}
\end{algorithm}

\section{Implementation}

We have developed \sys and all the baseline systems discussed in \S~\ref{sec:eval-setup} using Pytorch, amounting to 7377 lines of Python code. To enhance the performance of delta encoding, we have implemented Huffman encoding and run-length encoding using 300 lines of C++ code. \sys can be effortlessly integrated into user applications with less than 10 lines of code modifications (\cref{lst:code-torch}). Additionally, we provide a custom callback for the MosaicML composer library, enabling single-line integration of our system into applications developed with Composer (\cref{lst:code-composer}). This strategy can be readily extended to frameworks such as Pytorch-lightning and Keras, both of which offer a high-level trainer interface.

\begin{lstlisting}[style=PyStyle,caption={\small \sys API usage with native Pytorch applications},label={lst:code-torch},]
from inshrinkerator import CompressorRegistry
# Before train loop
compressor = CompressorRegistry.get_compressor(
    system, model, eval_batches, search_config, search_metric, threshold)
...
while global_step < max_training_steps:
    ...
    # Inside train loop 
    if should_checkpoint(global_step):
        compressor.compress(global_step)
    ...
    # Before flushing gradients
    compressor.before_gradient_flush()
    optimizer.zero_grad()
    ...
    # After backwards pass
    compressor.on_backward_pass_end()
    optimizer.step()
    ...   
\end{lstlisting}

\begin{lstlisting}[style=PyStyle,caption={\small 
 MosaicML Composer application with \sys callback integration.},label={lst:code-composer}]
from composer import Trainer
from inshrinkerator.integrations.composer import InshrinkeratorCallback
# Add InshrinkeratorCallback to list of callbacks
trainer = Trainer(
    ...
    callbacks=[InshrinkeratorCallback()],
)
trainer.fit()
\end{lstlisting}

\section{Evaluation Setup}
\label{sec:eval-setup}
In this section, we describe the setup used in our evaluations. Our experimental environment comprises servers equipped with 8 NVIDIA A40s GPUs connected with peer-wise NVLINK, 128 AMD CPU cores, and 504 GB memory. 

We compare checkpoint compression systems along the following three metrics:

\noindent\textbf{Storage Efficiency.} We report the end-to-end compression ratio, calculated by comparing the total size of compressed model checkpoints to their uncompressed counterparts.

\noindent\textbf{Model Quality.} We compare the quality of machine learning models trained from compressed checkpoints against a baseline model trained without checkpoint compression and report the \emph{relative} quality degradation on accuracy or loss.%

\noindent\textbf{Runtime Overhead.} We report time spent on compression as a fraction of the total training time.%

\begin{table}[t]
\centering

\begin{tabular}{l s s @{}}\\
\toprule
Model  & Task & Parameters \\
\midrule
ResNet18 [\citenum{resnet}] & Vision & 11.7 M \\
ResNet50 [\citenum{resnet}] & Vision & 25.6 M \\
ResNet152 [\citenum{resnet}] & Vision & 60.2 M \\
MobileNet V3L [\citenum{mobilenet}]  & Vision & 5.5 M \\
VGG19 [\citenum{vgg}] & Vision & 143.7 M \\
ViT B32 [\citenum{vit}] & Vision & 88.2 M \\
ViT L32 [\citenum{vit}] & Vision & 306.5 M \\
BERT Base [\citenum{bert}] & Language & 110 M \\
BERT Large [\citenum{bert}]  & Language & 345 M \\
GPT-2 Medium [\citenum{gpt2}] & Language & 335 M \\
Pythia 1B [\citenum{pythia}] & Language & 1 B \\
\bottomrule
\end{tabular}
\vspace{0.5em}
\caption{\small Summary of models used in the evaluation.\amey{add more precision to pythia number}}
\label{tab:models}
\end{table}

\subsection{Evaluation Scenarios, Models and Datasets}
We evaluate \sys across 7 different model families, including tasks in vision and language modeling. These models vary in complexity, with the number of parameters ranging from 5.5 million to 1 billion. Details on all models and datasets used are reported in \cref{tab:models} and \cref{tab:datasets}.

\begin{table}[t]
\centering
\small

\begin{tabular}{m{2cm} L r @{\hspace{0px}}}\\
\toprule
Dataset & Task & Training Mode \\
\midrule
Imagenet [\citenum{imagenet}] & Image Classification & Pre-training \\
C4~[\citenum{c4}] & Masked Language modeling & Pre-training \\
Openwebtext~[\citenum{openwebtext}] &  Next Token Prediction & Pre-training \\
STS-B~[\citenum{glue}] & Semantic Textual Similarity & Transfer Learning \\
MNLI~[\citenum{glue}] & Natural Language Inference & Transfer Learning \\
Alpaca~[\citenum{alpaca}] & Instruction Fine-Tuning & Transfer Learning \\
\bottomrule
\end{tabular}
\caption{\small Summary of datasets used for evaluation}
\label{tab:datasets}
\end{table}

We evaluate \sys for two use cases: 

\noindent\textbf{Fault-tolerant Training.} We simulate multiple failures during the training process and perform the recovery using compressed checkpoints. Our evaluation involves four models: ResNet-152, ViT-32L, BERT Base, and GPT2 Medium. The first two, representing the vision models, are trained on the Imagenet 1K datasets for classification tasks. The latter two, BERT and GPT models, undergo training for language modeling tasks using C4 and OpenWebText datasets, respectively. Training of GPT2 Medium and ViT-32L requires $\approx$ 650 NVIDIA A40 GPU hours. The hyper-parameters and implementation details are provided in \cref{app:model-imps}.

\noindent\textbf{Transfer Learning.} We perform transfer learning tasks by training from a compressed checkpoint of pre-trained models. We use BERT Large and Pythia models. The BERT model is fine-tuned on GLUE  tasks, STS-B, SST-2, and MNLI, while the Pythia model is fine-tuned using the Alpaca dataset.%

\subsection{Baseline Systems}

We compare the performance of \sys against two state-of-the-art checkpoint compression systems, Check-N-Run~\cite{checknrun} and QD-compressor~\cite{QD}, as well as a post-training quantization system GOBO~\cite{gobo} that uses a similar quantization technique. Due to a lack of openly accessible functional implementations, we re-implemented their described quantization and delta compression algorithms within the \sys framework. All re-implementations benefit from GPU-accelerated quantization algorithms. None of the systems supports dynamic configurations, so we extended \sys's search capability to all baselines, with each baseline having a different search space defined by its quantization algorithm. Unless otherwise mentioned, we use a default per-checkpoint quality degradation threshold of $\epsilon=0.05$ in the configuration search. \cref{all:baseline-imps} provides additional implementation details.

\begin{itemize}
\item \QD: An entropy-based variable bit-width uniform quantization scheme and delta compression with Huffman encoding, used in QD-compressor. 
\item \CNR: Adaptive uniform quantization \cite{adaptive-uniform-quant} algorithm used in Check-N-Run. The delta compression algorithm was designed for recommendation models and does not directly extend to general DL workloads.
\item \GOBO: Naive \kmeans non-uniform quantization with L1 distance used in GOBO. GOBO does not have a delta compression algorithm. 
\end{itemize}

\section{Evaluation}
\label{sec:evaluation}

In this section, we evaluate the empirical performance of \sys. Experiments show that: 

\begin{itemize}
\item \sys achieves up to 26-39× compression ratio and less than 0.6\% runtime overhead in fault-tolerant training, consistently offering the best accuracy-storage tradeoff compared to alternatives (\S~\ref{sec:eval-fr}).
\item \sys achieves better or similar performance on downstream transfer learning tasks using compressed checkpoints that are 10$\times$ smaller than the size of the pre-trained models, outperforming the closest state-of-the-art compression system by 2$\times$ (\S~\ref{sec:eval-contlearn}).
\item \sys's non-uniform quantization algorithm and delta encoding scheme contribute meaningfully to the overall performance (\S~\ref{sec:ablations}).
\end{itemize}

\begin{figure*}[ht]
    \centering
    \includegraphics[width=0.99\textwidth]{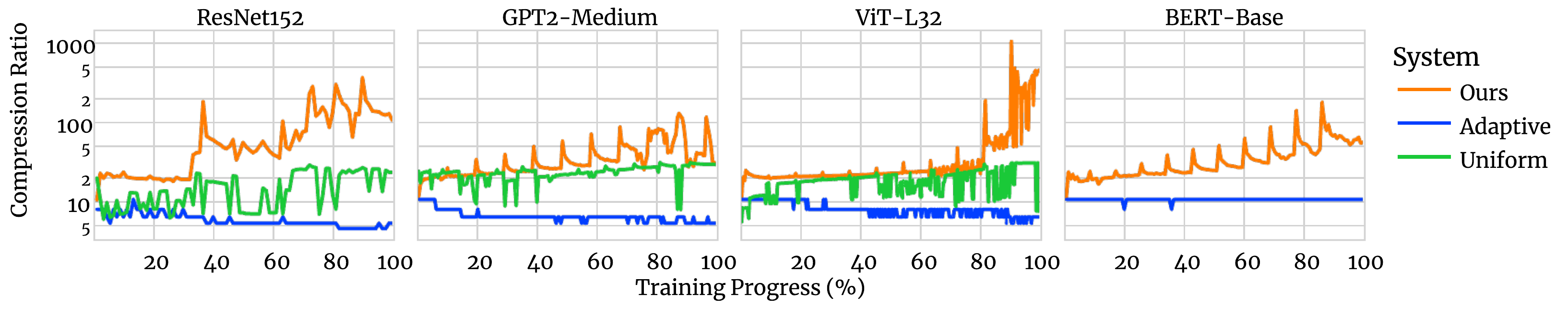}
    \caption{\small
        Compression ratio at every checkpoint during the training of various models.}
    \label{fig:cr-timeline}    
\end{figure*}

\subsection{End-to-end Evaluation: Fault-Tolerant Training}
\label{sec:eval-fr}

We conduct end-to-end training of four models under simulated failure conditions. Each experiment has 10 evenly distributed failures (every 3-8 hours of training) throughout the training duration. All models are trained using data parallelism across 8 NVIDIA A40 GPUs, with the largest models taking 3.5 calendar days (27 GPU days).

\cref{tab:failure-recovery} compares different systems in terms of model quality, compression ratio, and runtime overhead on the failure recovery task. Overall, \sys achieves between 26-39$\times$ reduction in storage requirements across various model families, marking a 3.5-6.5$\times$ improvement over \CNR, and 1.3-3.3$\times$ improvement over \QD. This is achieved while keeping the end-to-end relative degradation below 1\%, and the overhead of our system remains less than 0.6\% of the total training time. For ViT-L32, we observe that restoring from compressed checkpoints can act as a regularization mechanism, and result in better end-to-end performance.

While \QD achieves a higher compression ratio compared to \CNR due to its delta compression ability, we find the quality of models generated by the system to be inconsistent. In particular, we see a large drop in accuracy (to 4.43\%) after 8 restores for BERT-Base training. \GOBO is omitted in this experiment as it was originally designed for post-training quantization and is significantly slower than other in-training methods (details in \S~\ref{sec:ablations}).

\cref{fig:cr-timeline} shows the change of the compression ratio throughout training, where the peaks in the compression ratio correspond to checkpoints directly following restores from the compressed state. \sys's delta encoding scheme is particularly effective during the later stages of training when model updates decelerate. \sys's peak compression performance is above 100$\times$ for all models.

\begin{table}[t]
\small
\centering
\begin{tabular}{lllll}%
\toprule
& System & Deg. \% & CR & Ovr. \%  \\
\midrule
ResNet-152 & \CNR & 0.36 & 5.71 & 0.61 \\
(Base Acc: 77.41\% )& \QD & \textbf{-0.09} & 11.82 & \textbf{0.11} \\
& \ours & 0.50 & \textbf{39.09} & 0.35 \\
\midrule
ViT L32 & \CNR & -0.21 & 7.82 & 0.96 \\
(Base Acc: 71.36\%) & \QD & 0.12 & 16.31 & \textbf{0.18} \\
& \ours & \textbf{-0.46} & \textbf{26.19} & 0.51 \\
\midrule
BERT Base & \CNR & 3.97 & 6.78 & 1.03 \\
(Base Acc: 68.49\%)& \QD & \orangecross & \orangecross & \orangecross \\
& \ours & \textbf{0.80} & \textbf{29.25} & \textbf{0.58} \\
\midrule
GPT-2 M & \CNR & \textbf{0.25} & 6.35 & 0.55 \\
(Base Loss: 2.72) & \QD & 1.10 & 22.15 & \textbf{0.10} \\
& \ours & 0.61 & \textbf{29.81} & 0.32 \\
\bottomrule
\end{tabular}
\vspace{0.5em}
\caption{\small 
    Performance of checkpoint compression systems for failure-recovery. 
    Deg. \% = Relative quality degradation, CR = Compression ratio, Overhead \% =  Runtime Overhead.
}
\label{tab:failure-recovery}
\vspace{-1.5em}
\end{table}

\cref{fig:fr-tradeoff-single} illustrates the trade-off between storage overhead and model quality under different quality thresholds ($\epsilon$) in ResNet-152 training. At $\epsilon=0.01$, we achieve 21.82$\times$ compression, with no end-to-end quality degradation. At $\epsilon=0.05$, there is a small end-to-end relative accuracy drop (0.5\%) while the compression ratio goes up to 39.09$\times$. For all thresholds, \sys demonstrates an end-to-end degradation of less than 1\% and achieves a better compression-quality tradeoff when compared to \CNR and \QD. Notably, both baselines struggle to reach high compression ratios.

Finally, \sys can withstand up to 20 restores for ResNet152, which is equivalent to a restore every 4.5 epochs, while still maintaining a relative error of 1.1\%; for reference, the error is only 0.16\% for 5 restores.

\begin{table}[htb]
\centering
\begin{tabular}{@{\hspace{0px}} l @{\hspace{0px}} l *{5}{@{\hspace{2px}}c} @{\hspace{0px}}}
\toprule
Model & Task & Metric & Base & Adap & Uni & \ours \\
\midrule
\multirow{3}{*}{BERT--L} & & CR & - & 4.57 & 4.00 & \textbf{11.32} \\
& SST-2 & Acc & 93.2 & 92.9 & 93.0 & \textbf{93.3} \\
&  MNLI & Acc & 85.9 & \textbf{86.0} & \textbf{86.0} & 85.9 \\
& STS-B & PC & 86.1 & 87.6 & 87.8 & \textbf{88.3} \\
\midrule
\multirow{2}{*}{Pythia 1B} & & CR & - & 5.33 & 4.57 & \textbf{9.99} \\
& Alpaca & CE & 0.894 & 0.910 & \textbf{0.906} & 0.919 \\
\bottomrule
\end{tabular}
\vspace{0.5em}
\caption{\small Performance comparison on the transfer learning task.  CR = Compression ratio, Acc = Accuracy, PC = Pearson correlation, CE = Cross-entropy Loss.}
\vspace{-1em}
\label{tab:transfer-learning}
\end{table}

\subsection{End-to-end Evaluation: Transfer Learning}
\label{sec:eval-contlearn}
For the transfer learning experiments, we perform the compression process on pre-trained model checkpoints, then deploy the compressed models for downstream task fine-tuning. The checkpoints are obtained with a quality threshold $\epsilon=0.05$.  We evaluate the BERT-Base and Pythia 1B models on a suite of downstream tasks. \cref{tab:transfer-learning} shows results for all three fine-tuning tasks. We observe that for both the models, \sys can achieve compression ratios up to 11$\times$ even without delta encoding, ~2$\times$ higher compared to other systems. Notably, we find that the use of compressed checkpoints does not harm the performance of the downstream model, and we achieve performance comparable to baseline across all tasks and models.

\subsection{Microbenchmarks \& Ablation Studies}
\label{sec:ablations}

\begin{table*}[t]
\centering
\begin{tabular}{l *{3}{*{5}{c}}}
\toprule
Model & \multicolumn{12}{c}{Mean Number of Quantization Bins ($\downarrow$)}  \\
 & \multicolumn{4}{c}{0.1\% Degradation} & \multicolumn{4}{c}{1\% Degradation} & \multicolumn{4}{c}{5\% Degradation} \\
 \cmidrule{2-5}\cmidrule(lr){6-9}\cmidrule{10-13}
 & CNR+ & QD+ & GB+ & \ours & CNR+ & QD+ & GB+ & \ours & CNR+ & QD+ & GB+ & \ours  \\
\midrule
ResNet18 & 42.84 & \orangecross & \orangecross & \textbf{21.25} & 33.75 & 35.51 & 33.12 & \textbf{16.85} & 23.55 & 23.25 & 17.50 & \textbf{10.19} \\
ResNet50 & 46.36 & \orangecross & \redcross & \textbf{21.39} & 35.07 & 53.65 & 31.78 & \textbf{14.36} & 23.06 & 33.69 & 20.12 & \textbf{8.24} \\
ResNet152 & 59.32 & \orangecross & \redcross & \textbf{20.18} & 40.71 & \orangecross & \redcross & \textbf{11.92} & 26.74 & 39.36 & \redcross & \textbf{7.20} \\
MobileNet V3 Large & 65.24 & \orangecross  & \orangecross & \textbf{27.92} & 54.44 & \orangecross & \orangecross & \textbf{32.48} & 36.68 & 47.96 & 30.79 & \textbf{17.91} \\
VGG19 & 50.46 & \orangecross & \redcross & \textbf{13.06} & 31.01 & 45.36 & \redcross & \textbf{8.99} & 18.48 & 25.15 & \redcross & \textbf{6.26} \\
ViT B32 & 37.01 & \orangecross & \redcross & \textbf{19.88} & 28.37 & 38.43 & \redcross & \textbf{14.33} & 20.85 & 28.79 & \redcross & \textbf{9.75} \\
ViT L32 & 37.40 & \orangecross & \redcross & \textbf{13.31} & 32.95 & \orangecross & \redcross & \textbf{10.79} & 25.36 & 35.34 & \redcross &  \textbf{8.46} \\
BERT Base & \orangecross & \orangecross & \redcross & \textbf{54.30} & 81.45 & \orangecross & \redcross & \textbf{22.75} & 43.18 & 53.10 & \redcross & \textbf{12.74} \\
GPT-2 Medium & 106.86 & \orangecross & \redcross & \textbf{54.71} & 47.71 & 60.17 & \redcross & \textbf{16.19} & 24.60 & 31.14 & \redcross & \textbf{9.08} \\

\bottomrule
\end{tabular}
\vspace{0.5em}
\caption{
    \small One-shot compression performance under quality degradation constraints. We consider the mean number of quantization bins required to achieve the quality degradation constraint across ten evenly spaced checkpoints during training. System timeouts (30 mins for each evaluation of each configuration) are represented by \redcross, while scenarios, where no suitable configuration is found, are marked by \orangecross.}
\label{tab:one-hot-small}
\end{table*}

\begin{table}[htbp]
\centering
\begin{tabular}{@{} l *{4}{@{\hspace{6px}}c} @{}}
\toprule
Model & \multicolumn{4}{c}{Runtime in Milliseconds ($\downarrow$)} \\
 & \multicolumn{2}{c}{Quantile} & \multicolumn{2}{c}{Clustering (K=32)}  \\
\cmidrule{2-3}\cmidrule(lr){4-5}
& CuPy [\citenum{cupy}] & \ours & CuML [\citenum{cuml}] & \ours \\
\midrule
ResNet152 & 71.9 & \textbf{15.2} & 10200 & \textbf{1160} \\
ViT L32 & 364 & \textbf{80.9} & 50100 & \textbf{1140} \\
BERT Base & 155 & \textbf{33.9} & 20700 & \textbf{411} \\
GPT-2 M & 390 & \textbf{86.2} & 53400 & \textbf{819} \\
\bottomrule
\end{tabular}
\vspace{0.5em}
\caption{\small Runtime of various operations in \sys compared with off-the-shelf implementations.}
\label{tab:runtime-abalation}
\vspace{-1em}
\end{table}

\noindent\textbf{{Quality of Non-uniform Quantization.}} To evaluate different quantization algorithms, we collect ten checkpoints spaced throughout the training of different models. For each algorithm, we determine the minimum number of quantization bins needed to meet per-checkpoint thresholds of $\epsilon=0.001, 0.01, 0.05$, then calculate the mean number of bins across all checkpoints. We disable parameter pruning in our system for this experiment to allow for easier comparison. The results are presented in  \cref{tab:one-hot-small}. 

We observe that \sys consistently outperforms across all the models and quality thresholds. \GOBO, which uses a na\"ive implementation of k-means clustering algorithm for quantization, takes over 30 minutes for models with more than 30 million parameters, while \sys usually takes a few seconds. \QD fails to find valid quantization configurations for strict quality thresholds. Among baselines, only \CNR manages to constantly find quantization configurations that satisfy the quality thresholds. However, it requires a 2-4$\times$ higher number of quantization bins to achieve the same quality as \sys due to their use of uniform quantization.

\noindent\textbf{{Quantization Runtime.}} \cref{tab:runtime-abalation} reports the runtime of two key operations in our quantization algorithm: computing qualities for pruning and protection, and the \kmeans clustering for quantization. We implement our approach in Pytorch with support for both CPU and GPU execution. We compare the performance of our implementation against off-the-shelf GPU-enabled implementations for quantile estimation in CuPy and \kmeans clustering in CuML respectively. Our quantile estimation algorithm achieves 3-4$\times$ higher performance compared to CuPy, while we note a speed-up of up to 65$\times$ in quantile computation.

\begin{figure}
    \centering
    \includegraphics[width=0.9\linewidth]{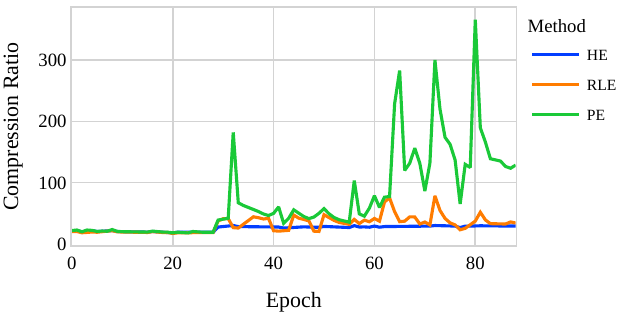}
    \caption{
    \small 
        Parameter rearrangement significantly improves the performance of delta encoding schemes (ResNet-152).
        HE = Huffman Encoding, RLE = Run Length Encoding + Huffman Encoding, PE = Parameter Rearrangement-based Encoding.}
    \label{fig:delta-abalation}
    \vspace{-1em}
\end{figure}

\noindent\textbf{Parameter Rearrangement-based Delta Encoding.} 
The key optimization in \sys's delta encoding scheme is parameter rearrangement, which enables it to decrease the entropy in each compressed chunk. 
To assess the contribution of this optimization, we compare three variants of our delta encoding scheme: the Huffman encoding variant used in \QD (HE), run-length plus Huffman encoding (RLE), and parameter rearrangement-based delta encoding (PE) on the ResNet-152 training job. The results are illustrated in \cref{fig:delta-abalation}. With parameter rearrangement, we achieve much higher compression in the later phase of training. Notably, we obtain a maximum compression ratio of 366 with rearrangement compared to 79 with RLE and 30 with HE. This translates to overall lower storage overhead. The Huffman and run-length encoding-based schemes only achieve overall storage overhead reduction of 25$\times$ and 28$\times$ respectively, as opposed to 39$\times$ with parameter rearrangement.

\section{Discussion}
\label{sec:discussion}

\sys demonstrates significant reduction in checkpoint storage overhead with minimal impact on model accuracy. However, several aspects of the system design warrant further investigation and development:

\begin{itemize}[leftmargin=*]
    \item \textbf{Delta Compression Optimization:} While delta compression optimizes storage, it can increase restore latency as the chain of delta checkpoints grows. Future work should explore adaptive policies for balancing storage efficiency and restore time, potentially incorporating periodic full checkpoints or intermidate compaction. These policies could dynamically adjust based on training progress, failure rates, and available resources.

    \item \textbf{Integration with High-Frequency Checkpointing:} Recent work \cite{checkfreq} has proposed dynamic, high-frequency, iteration-level checkpointing to minimize work loss due to failures. As model sizes increase, the storage and network costs of such frequent checkpoints become significant. Future research could explore how \sys can be integrated with these systems to alleviate storage and network bottlenecks while maintaining the benefits of frequent checkpointing. This integration may require adapting \sys's configuration search for short intervals between checkpoints and considering the impact on I/O bandwidth and overall system performance.

    \item \textbf{Adaptive Application-Specific Tuning:} The quality degradation threshold $\epsilon$ should be dynamically tuned based on application requirements and environmental factors. Future work could develop a framework for automatically determining optimal $\epsilon$ values using characteristics such as model architecture, dataset properties, and training environment. Online learning techniques could be employed to adapt $\epsilon$ during training, balancing compression efficiency with application-specific quality requirements.

    \item \textbf{Quantization Heuristics:} Our current approach of monotonically increasing quantization precision during training prioritizes model quality. Future work could explore more sophisticated strategies, to optimize the quantization strategy throughout training. Multi-objective optimization approaches could be developed to balance compression ratio, model quality, and computational overhead dynamically.
\end{itemize}

These considerations open avenues for further research in adaptive checkpoint compression strategies, potentially leading to even more efficient and versatile checkpoint storage solutions for large-scale deep learning training.

\section{Related Work}
\label{sec:related}

Our work intersects with three main areas of research: model compression techniques, checkpoint compression systems, and checkpointing at scale. We discuss each of these below.

\subsection{Model Compression Techniques} Model compression techniques aim to reduce the storage, memory, and computational requirements of deep learning models without compromising their performance. %

\minihead{Quantization} Quantization has been widely studied, particularly for efficient model inference. Post-Training Quantization (PTQ) \cite{banner2019post, cai2020zeroq, fang2020post, lowBitQuantization, He2018ECCV} and Quantization-Aware Training (QAT) \cite{courbariaux2015binaryconnect, gysel2016hardware, NIPS2016d8330f85, lin2015neural, rastegari2016xnor, tailor2020degree} are two main approaches. While QAT incorporates quantization errors into the loss function during training, our proposed in-training checkpoint quantization operates on intermediate model states without affecting the training optimization function.

Variable bit-width quantization assigns different bit-widths to model layers based on their importance \cite{faghri2020adaptive, kadambi2020comparing, choi2016towards, NEURIPS2020d77c7035, dong2019hawq, li2018optimization}. This approach enables better compression without compromising performance, although determining the optimal bit-width for each layer can be  time-consuming.

\minihead{Pruning} Pruning is another popular model compression technique in which unimportant parameters are removed from the model. Magnitude-based pruning techniques remove parameters with the smallest absolute values \cite{han2015learning}, while sensitivity-based pruning methods consider the impact of parameter removal on the model's loss \cite{sen-5, sen-6, sen-7, sen-1}.

Recent work on the Lottery Ticket Hypothesis with Rewinding \cite{lth} has shown that simple magnitude-based heuristics can identify extremely sparse subnetworks early in training that yield performance similar to the full network when trained in isolation. Our work leverages this insight for in-training pruning, applying it to checkpoints during training.

\subsection{Checkpoint Compression Techniques}

Several studies have explored compression techniques specifically for model checkpointing systems to minimize storage overhead \cite{QD,lcCompression,checknrun,deltadnn}. These systems typically employ a combination of model quantization and delta encoding, leveraging the high similarity between adjacent checkpoints.

\minihead{Quantization in Checkpoint Compression} Recent checkpoint compression systems have employed various quantization techniques. LC-Checkpoint \cite{lcCompression} uses an exponent-based non-uniform quantization scheme, choosing the non-linear function heuristically. Delta-DNN \cite{deltadnn} proposes a uniform quantization approach with an exhaustive configuration search strategy. QD-Compressor \cite{QD} employs an entropy-based variable bit-width uniform quantization scheme, while Check-N-Run \cite{checknrun} applies an adaptive uniform quantization method \cite{adaptive-uniform-quant} for large deep recommendation models.

\minihead{Delta Compression} Delta compression is a key technique in reducing storage overhead for model checkpoints. Each compressed checkpoint only stores the differences since the previous checkpoint, allowing for reconstruction of the final checkpoint by iteratively applying these deltas.

While QD-Compressor \cite{QD} proposes a delta encoding algorithm, it doesn't account for changes in quantization bit widths induced by their variable bit-width technique. Check-N-Run \cite{checknrun} introduces a delta compression scheme specifically for recommendation models, but its applicability to traditional deep learning models is limited.

\subsection{Checkpointing at Scale} As machine learning models grow in size and complexity, efficient checkpointing strategies become crucial.

\minihead{Distributed Checkpointing} Distributed checkpointing systems \cite{distChkpt-1, checknrun} save partial model states across multiple nodes within a distributed system. This approach is especially useful when the model state cannot fit on a single device and reduces the overhead associated with storing and retrieving large checkpoints. Our proposed method is fully compatible with distributed checkpointing, enabling checkpoint compression even when the model state is distributed across different processes.

\minihead{Frequent Checkpointing} Frequent checkpointing minimizes the risk of losing progress due to unexpected failures and allows for more granular control over the training process. CheckFreq \cite{checkfreq} facilitates frequent checkpointing through a two-phase process that overlaps computation with checkpointing operations. It dynamically determines checkpointing intervals to balance trade-offs between overhead and wasted work. %

\section{Conclusion}
\label{sec:conclusion}

\sys combines non-uniform quantization, dynamic quantization configuration search, and quantization-aware delta compression to enable efficient in-training checkpointing. In our experiments, \sys demonstrates the ability to induce a tradeoff space between the end-to-end model quality and checkpoint storage overhead, achieving up to \textit{10$\times$ compression without loss in model quality} for fault-tolerant training and transfer learning.

\section*{Acknowledgments}
This material is based upon work partially supported by the National Science Foundation under grant number CNS-2420977 and IIS-2335881.
We would also like to express our sincere gratitude to the reviewers, the PC panel, and especially our shepherd Dr. 
Saurabh Bagchi for their insightful comments and thoughtful consideration, which significantly improved the quality of this paper. \\
\textbf{Disclaimer}: Any opinions, findings, and
conclusions or recommendations expressed in this material are those of the authors and do not necessarily
reflect the views of the National Science Foundation.

\bibliographystyle{ACM-Reference-Format}
\bibliography{main}

\clearpage
\appendix

\section{Evaluation Setup}

\subsection{Model Implementations}
\label{app:model-imps}

For pre-training the ResNet-152 and ViT-32L models on Imagenet1k, we utilize the implementation and hyperparameters available in the Torchvision package \cite{torchvision}, while BERT-Base and GPT-2 Medium training leverages implementations provided by MosaicML \cite{mosiac-bert} and NanoGPT \cite{nangpt} repositories respectively. Notably, we maintain the original training duration and model hyperparameters, avoiding any potential skewing of evaluation results, with the exception of GPT-2 Medium where we truncate the training procedure to 30 billion tokens to restrict the training duration. In ViT-32L we find that the default parameters result in unstable training, so we reduce the learning from 0.003 to 0.001.

\subsection{Baseline Implementation}
\label{all:baseline-imps}

In this section, we describe modifications made to Check-N-Run~\cite{checknrun}, QD-compressor~\cite{QD}, and GOBO~\cite{gobo} algorithms during our implementation of the baselines.

For QD-compressor, we address an oversight in the original description which restricted the applicability of the delta compression algorithm to cases where the number of bins assigned to a layer changes between successive checkpoints.

While GOBO was originally proposed for post-training quantization, we implement it within the \sys framework to allow in-training compression. We also modify the GOBO clustering algorithm to include a soft termination condition which significantly reduces the runtime.

Since Check-N-Run was originally developed with the intent to support deep learning recommendation models, the delta compression algorithm proposed in the system cannot be used for regular deep learning workloads. Thus we only implement the quantization algorithm \cite{adaptive-uniform-quant} used by Check-N-Run for comparison.

Recall that we also added \sys's dynamic configuration search capability for all baselines. \cref{tab:search-space,tab:QDsearch-space,tab:CNRsearch-space,tab:GOBOsearch-space} list the range of values used for the configuration parameters for \sys, QD, CNR, and GOBO respectively.

\begin{table}[htpb]
\centering
\small
\begin{tabular}{l l}\\
\toprule
Parameter & Values \\
\midrule
Number of bins & 4, 6, 8, 12, 16, 32 \\
Number of bins (Embedding Layers) & 16, 32 \\
Pruning Fraction & 0, 0.1, 0.2, 0.3, 0.4, 0.5 \\
Pruning Metric & Magnitude, Sensitivity \\
Protection Fraction & 0.0005, 0.005, 0.01 \\
\bottomrule
\end{tabular}
\caption{Search space for \sys quantization configuration}
\label{tab:search-space}
\end{table}

\begin{table}[htbp]
\centering
\small
\begin{tabular}{l l}\\
\toprule
Parameter & Values \\
\midrule
Minimum number of bins & 8, 16, 32, 64, 128, 256 \\
Maximum number of bins & 8, 16, 32, 64, 128, 256 \\
\bottomrule
\end{tabular}
\caption{Search space for QD quantization configuration}
\label{tab:QDsearch-space}
\end{table}

\begin{table}[htbp]
\centering
\small
\begin{tabular}{l l}\\
\toprule
Parameter & Values \\
\midrule
Number of bins & 8, 16, 32, 64, 128, 256 \\
Number of step bins & 10, 25, 50, 100 \\
Range & 0.1, 0.2, 0.3, 0.4, 0.5 \\
\bottomrule
\end{tabular}
\caption{Search space for CNR quantization configuration}
\label{tab:CNRsearch-space}
\end{table}

\begin{table}[htpb]
\centering
\small
\begin{tabular}{l l}\\
\toprule
Parameter & Values \\
\midrule
Number of bins & 8, 16, 32, 64, 128, 256 \\
Number of bins (Embedding Layers) & 16, 32 \\
Outlier threshold & -4 \\
Max Iteration & 1000 \\
\bottomrule
\end{tabular}
\caption{Search space for GOBO quantization configuration}
\label{tab:GOBOsearch-space}
\end{table}

\subsection{Additional Details for Approximate K-Means}\label{app:kmeans}

\subsubsection{Proof for Error-bound Guarantees}

We restate the notations and propositions here for clarity. Let $X= \{ x_1, \ldots, x_n\}$ and $\tX = \{ \tilde{x_1}, \ldots, \tilde{x_n}\}$ denote a set of $n$ points where $x_i, \tx_i \in [0,1]$ and for all $i$, $|x_i - \txi| \leq \alpha x_i$ for some $\alpha \in (0,1)$. Let $M = \{\mu_1, \ldots, \mu_k\}$ be the cluster centers of $X$ and the loss of $M$ over inputs $X$ is defined as $\loss(M) = \frac{1}{n}\sum_{x \in X} \min_{\mu \in M} |x - \mu|$. The cluster centers $\tM$ and loss $\qloss$ over $\tX$ are defined similarly. The optimal loss $\loss(\oM)$ is the loss of the optimal cluster centers $M^{*} = \argmin_{M} \loss(M)$. The optimal loss $\qloss(\otM)$ and cluster centers $\otM$ over $\tX$ are defined similarly. Let $\mu_{M}(i) = \argmin_{\mu} |\mu - x_i|$ denote the cluster center in $M$ that is closest to the ith input $x_i$ ($\tmu_{\tM}(i)$ is defined similarly for $\txi$).

Note that the loss $\loss(\tM)$ over the original inputs $X$ with any clustering $\tM$ over $\tX$ is at most $\alpha$ greater than the loss $\qloss(\tM)$,

\[   
    \qloss(\tM) = \frac{1}{n} \sum_{i=1}^n |\tmu_{\tM}(i) - \txi|  \leq \frac{1}{n} \sum_{i=1}^n (|\tmu_{\tM}(i) - x_i| + \alpha x_i) 
\]

\[
 = \loss(\tM) + \frac{\alpha}{n} \sum_{i=1}^n x_i \leq \loss(\tM) + \alpha
\]

where the first inequality follows from triangle inequality and the second follows from the fact that $\frac{1}{n}\sum_{i=1}^n x_i \leq 1$. Using similar arguments, one can verify that $\loss(\tM) \leq \qloss(\tM) + \alpha$.

\begin{proposition}
Given any two sets $X, \tX$ of $n$ points as defined above such that $|x_i - \txi| \leq \alpha x_i$ $\forall i$. For any $k \geq 1$, the difference between the optimal loss of clusterings over $X$ and $\tX$ are bounded as, 

\begin{equation}
    |\loss(\oM) - \qloss(\otM)| \leq \alpha
\end{equation}
\end{proposition}

\begin{proof} We first show that $\qloss(\otM) \leq \loss(\oM) + \alpha$.
   \[
   \qloss(\otM) = \frac{1}{n} \sum_{i=1}^n |\tmui - \txi| \leq \frac{1}{n} \sum_{i=1}^n |\mui - \txi| 
   \]
   \begin{equation}\label{Eq:boundsize1}
   \leq \frac{1}{n} \left ( \sum_{i=1}^n |\mui - x_i| + \alpha \sum_{i=1}^n x_i \right) \leq \loss(\oM) + \alpha
   \end{equation}

In \cref{Eq:boundsize1}, the first inequality follows from the fact that $\otM$ is the optimal clustering over $\tX$ and hence replacing the cluster centers $\tmui$ with $\mui$ leads to a quantity which is greater than equal to $\qloss(\otM)$. The next inequality follows from triangle inequality and uses the fact that $|x_i - \txi| \leq \alpha x_i$. Using similar arguments, one can verify that $\loss(\oM) \leq \qloss(\otM)  + \alpha$. Combining the two equations, we have that $|\loss(\oM) - \qloss(\otM)| \leq \alpha$ which is the main statement of the proposition.

The proof for $\loss(\oM) \leq \qloss(\otM)  + \alpha$ follows similar arguments leading to the main statement of the result.

\end{proof}

\begin{remark} The results show that the difference between optimal loss of k-means clustering over inputs $X$ and its grouped counterpart $\tX$ is bounded by $\alpha$. Building upon this, we will show that when the clustering is performed with sample-weighted k-means++, the expected loss over $\tX$ is bounded by the optimal loss over $X$ and $\alpha^2$. 
\end{remark}

To keep notations close to \cite{kmeanspp}, we denote the loss over squared norm as $\phi(M) = \frac{1}{n}\sum_{x \in X} \min_{\mu \in M} |x - \mu|^2$. The loss $\phiq(M)$ is defined similarly over $\tX$: $\phiq(\tM) = \frac{1}{n}\sum_{\tx \in \tX} \min_{\tmu \in \tM} |\tx - \tmu|^2$. 

\begin{proposition}\label{prop2}
Given any two sets $X, \tX$ of $n$ points as defined above such that $|x_i - \txi| \leq \alpha x_i$ $\forall i$. For any $k \geq 1$, let $\wM$ be the clustering obtained by applying weighted kmeans++ (Algorithms \ref{algo:initialize} and \ref{algo:kmeans}) over the set $\tX$. Then,

\begin{equation}
    E[\phiq(\wM)] \leq 16 (\ln k + 2) (\phi(\oM) + \alpha^2)
\end{equation}

\end{proposition}
\begin{proof}
    We first upper bound the optimal loss $\phiq(\otM)$ with $\phi(\oM)$ and $\alpha$. See that

 \[
   \phiq(\otM) =  \frac{1}{n} \sum_{i=1}^n |\tmui - \txi|^2  \leq \frac{1}{n} \sum_{i=1}^n |\mui - \txi|^2
   \]
   \begin{equation}\label{Eq:squaredbound}
   \leq 2 \left ( \frac{1}{n}\sum_{i=1}^n |\mui - x_i|^2 + \frac{\alpha^2}{n} \sum_{i=1}^n x_i^2 \right)  \leq 2(\phi(\oM) + \alpha^2)
   \end{equation}

   Applying weighted-kmeans++ (Algorithms \ref{algo:initialize} and \ref{algo:kmeans}) over the set $\tX$ produces $\wM$. Using Theorem 1.1 in \cite{kmeanspp}, we have that $E[\phiq(\wM)] \leq 8 (\ln k + 2) (\phiq(\otM))$ and using Eq. \ref{Eq:squaredbound} to replace $\phiq(\otM)$, we obtain the main statement of the proposition.

\end{proof}

\begin{remark} The result implies that \cref{algo:initialize,algo:kmeans} is $O(\ln k)$-competitive with $\phi(\oM) + \alpha^2$ while being orders of magnitude faster in comparison with the vanilla k-means++ which is $O(\ln k)$-competitive with $\phi(\oM)$.
\end{remark}

We now analyze certain conditions where the clustering memberships of every point in the two cases ($X$ and $\tX$) would be identical. Let $C = \{X_1, \ldots, X_k\}$ be a clustering (or partition) of $X \subseteq [0, 1]$ with centers $M = \{\mu_1, \ldots, \mu_k\}$. For all $x \in X_i$, and $i \neq j$, $|x - \mu_i| \leq |x - \mu_j|$. Let the minimum distance between any two points in two different clusters in $C$ be the split of the clustering denoted as $split_C(X)$. Let the maximum distance between any two points in the same cluster be the width of the clustering denoted as $width_C(X)$.

\begin{definition}
   \cite{ben2014clustering} A clustering $C = \{X_1, \ldots, X_k\}$ of $X \subseteq [0, 1]$ is $\sigsep$ for $\sigma \geq 1$ if  $split_C(X) > \sigma \cdot width_C(X)$.
\end{definition}

In other words, any clustering $C$ is $\sigsep$ if the minimum distance between any two points in separate clusters is greater than the maximum distance between two points in the same cluster. 

We now show that if the optimal clustering over the set $X$ is $\sigsep$, then for small enough $\alpha$ depending on the value of $\sigma$, the optimal clustering over $X$ and $\tX$ will be the same. 

First note that, if a $\sigsep$ clustering $C$ exists for a given set of inputs, then only one such paritioning of the inputs. 

\begin{lemma}
    \cite{ackerman2009clusterability} If there exists a k-clustering $C$ of $X$, for $k \geq 2$, such that $C$ is $\sigsep$, then there is only one such partitioning of $X$.
\end{lemma}

The proof of the above Lemma can be found in \cite{ackerman2009clusterability}. We now show that if $\alpha < \frac{split_C(X) - width_C(X)}{4}$, then the same paritioning over $\tX$ is also $\sigsep$ and hence the optimal clustering for both $X$ and $\tX$ are identical.

Let $s_1$ and $s_2$ be the two points in $X$ such that the distance between them is minimum among all pairs of points in separate clusters ($|s_2 - s_1| = split_C(X)$). Let $w_l$ and $w_2$ be two points in $X$ such that the distance between them is maximum among all pairs of points in the same cluster.

For the points in $X$ to have a $\sigsep$ clustering in $\tX$, the following condition should hold in the worst case,

\[
 s_2 (1-\alpha) - s_1 (1+\alpha) > w_2 (1+\alpha) - w_1 (1 -\alpha)
\]
\[
\implies (s_2 - s_1) - \alpha (s_1 + s_2) > (w_2 - w_1) + \alpha (w_1 + w_2)
\]

\[
\implies \alpha < \frac{(s_2 - s_1) - (w_2 - w_1)}{4} = \frac{split_C(X) - width_C(X)}{4}
\]

This implies that if the original clusters are $\sigsep$ and $\alpha$ is small enough then the optimal cluster memberships over $X$ and $\tX$ will be identical.

\end{document}